%% file: main.tex
\documentclass{article}

\PassOptionsToPackage{numbers}{natbib}

 \usepackage[preprint]{main}

\usepackage[utf8]{inputenc} 
\usepackage[T1]{fontenc}    
\usepackage{hyperref}       
\usepackage{url}            
\usepackage{booktabs}       
\usepackage{amsfonts}       
\usepackage{nicefrac}       
\usepackage{microtype}      
\usepackage{xcolor}         

\usepackage{amsfonts}
\usepackage{amsmath}
\usepackage{amssymb}
\usepackage{mathtools}
\usepackage{amsthm}
\usepackage{multirow}

\usepackage{algorithm}
\usepackage{algorithmic}

\usepackage{pgfplots}              
\pgfplotsset{compat=1.18}
\usepgfplotslibrary{fillbetween}   

\usepackage{booktabs}       

\usepackage{subfigure}

\newtheorem{Theorem}{Theorem}
\newtheorem{Definition}{Definition}

\newtheorem{Lemma}{Lemma}

\title{
Tail-Risk-Safe Monte Carlo Tree Search under PAC-Level Guarantees
}

\author{
    Zuyuan Zhang\\
    The George Washington University\\
    \texttt{zuyuan.zhang@gwu.edu}\\
    \And
    Arnob Ghosh\\
    New Jersey Institute of Technology\\
    \texttt{arnob.ghosh@njit.edu}\\
    \And
    Tian Lan\\
    The George Washington University\\
    \texttt{tlan@gwu.edu}
}

\begin{document}

\maketitle

\input{01Abstract}

\input{02Intro}

\input{03Background}

\input{04Solution}

\input{05Experment}

\input{06Conclusion}

\medskip

{
\small
\bibliography{ref}
\bibliographystyle{unsrtnat} 
}

\input{07appendix}

\end{document}

%% file: 01Abstract.tex
\begin{abstract}
 Making decisions with respect to just the expected returns in Monte Carlo Tree Search (MCTS) cannot account for the potential range of high-risk, adverse outcomes associated with a decision. To this end, safety-aware MCTS often consider some constrained variants -- by introducing some form of mean risk measures or hard cost thresholds. These approaches fail to provide rigorous tail-safety guarantees with respect to extreme or high-
risk outcomes (denoted as tail-risk), potentially resulting in serious consequence in high-stake scenarios. This paper addresses the problem by developing two novel solutions. We first propose CVaR-MCTS, which embeds a coherent tail risk measure, Conditional Value-at-Risk (CVaR), into MCTS. Our CVaR-MCTS with parameter $\alpha$ achieves explicit tail-risk control over the expected loss in the "worst $(1-\alpha)\%$ scenarios." Second, we further address the estimation bias of tail-risk due to limited samples. We propose Wasserstein-MCTS  (or W-MCTS) by introducing a first-order Wasserstein ambiguity set $\mathcal{P}_{\varepsilon_{s}}(s,a)$ with radius $\varepsilon_{s}$ to characterize the uncertainty in tail-risk estimates. We prove PAC tail-safety guarantees for both CVaR-MCTS and W-MCTS and establish their regret.
Evaluations on diverse simulated environments demonstrate that our proposed methods outperform existing baselines, effectively achieving robust tail-risk guarantees with improved rewards and stability.
\end{abstract}

%% file: 02Intro.tex
\section{Introduction}
Monte‑Carlo Tree Search (MCTS) has delivered state‑of‑the‑art performance in a variety of sequential decision‑making domains, including board games~\cite{gelly2006modification,silver2016mastering,silver2018general}, real‑time strategy games~\cite{schrittwieser2020mastering}, and robotic or classical planning tasks~\cite{chaslot2008progressive,luo2024survey,zhang2025lipschitz}.  
Its success stems from a principled balance of exploration and exploitation—typically realized via UCT‑style selection rules—together with sample‑efficient simulation, expansion, and back‑propagation phases.
Standard MCTS, however, optimizes \emph{expected} return only; it is blind to rare but catastrophic cost or safety.  Ignoring such \emph{tail-risk} cost can be disastrous in safety‑critical settings, e.g.\ autonomous driving~\cite{hubmann2017decision,bouton2019cooperation,zhang2025learning,chen2022cluster,yang2021stability,yu2025look,chen2024design,zou2024distributed,li2024crowdsensing,xu2023cnn} or medical decision support~\cite{hayes2022montecarlotreesearch,DBLP:conf/icml/Castellini0ZSFS23,zhang2024modeling}, where a single high‑impact failure cost may outweigh many high successes.  Addressing this limitation calls for MCTS variants that reason explicitly about the extreme end of the return distribution on the cost in addition to the average reward.

Recent work on \emph{constrained} Markov decision processes (CMDPs) seeks to maximise expected reward while keeping the long‑run \emph{average} cost below a prescribed threshold.  Although useful, this risk‑neutral criterion offers no protection against low‑probability, high‑impact losses—the tail of the cost distribution that matters most in safety‑critical settings.  A few papers have begun to study \emph{risk‑constrained} MDPs, but current algorithms lack finite‑time performance guarantees and, crucially, have not been integrated with MCTS. Within the MCTS literature itself, several so‑called \emph{Constrained MCTS} (C‑MCTS) variants have emerged:  
\textbf{Threshold‑UCT}~\cite{kurevcka2025threshold} heuristically prunes branches whose cumulative cost exceeds a hard cap;  
\textbf{Cost‑Constrained MCTS}~\cite{lee2018monte} trains an offline “safety critic’’ to veto risky actions; and  
\textbf{Distributional MCTS}~\cite{hayes2022montecarlotreesearch} maximizes a non‑linear utility over the return distribution.  
Though useful,  these solutions mostly focus on some form of mean cost or hard cost thresholds, and did not provide rigorous tail-risk guarantees while ensuring sub-linear regret. 
In this paper, we are focusing on the following question: {\em Can we achieve a sub-linear regret with provable tail-risk satisfaction guarantee for risk-constrained MCTS?}

Estimating tail‑risk metrics—such as Conditional Value‑at‑Risk (CVaR)—from a finite number of samples is intrinsically challenging, because CVaR depends on the sparsely populated extreme quantiles of the return distribution. On top of that, a viable algorithm must strike a careful balance between maximizing reward and enforcing the risk constraint, further compounding the difficulty.

\textbf{First}, we introduce \emph{CVaR--MCTS}, which incorporates the coherent tail--risk metric \textit{Conditional Value--at--Risk} (CVaR)~\cite{rockafellar2002conditional} into the UCT search routine.  
For a user--specified confidence level $\alpha \in (0,1)$, CVaR evaluates the \emph{expected} cost \emph{conditioned} on being in the worst\,$(1-\alpha)\%$ of outcomes:
\[
\operatorname{CVaR}_{\alpha}(Z)
 \;=\;
 \mathbb{E}\!\bigl[\,Z \mid Z \,\ge\, \operatorname{VaR}_{\alpha}(Z)\bigr],
\]
where $Z$ denotes the random return (or cost) and $\operatorname{VaR}_{\alpha}(Z)$ is the $\alpha$--quantile.  
A smaller~$\alpha$ probes deeper into the tail, producing a more conservative assessment.  
Embedding this criterion within MCTS yields decisions that explicitly control extreme losses rather than merely optimizing the mean.

CVaR--MCTS employs \emph{online Lagrangian dual updates} to balance exploration, exploitation, and adherence to the $\text{CVaR}_{\alpha}$ constraint, capping the expected loss in the worst $(1-\alpha)\,\%$ of outcomes.  
We show that for any $\delta\in(0,1)$, after $T$ node expansions
$
\Pr\!\bigl[\text{CVaR}_{\alpha}(Z_T)\le\text{Threshold}\bigr]\;\ge\;1-\delta,$
and the algorithm attains a regret of at most $\tilde{\mathcal O}(\sqrt{T})$, providing finite‑time tail‑safety with sub‑linear regret guarantee.  

Unfortunately, CVaR-MCTS still relies on quantile estimates from limited samples; in early exploration or when the environment drifts, estimation bias may lead to temporary constraint violations. To curb finite‑sample bias, we surround each empirical return distribution with a
first‑order Wasserstein ball
$
\mathcal{P}_{\varepsilon_s}(s,a)=\bigl\{\widetilde P:W_1(\widetilde P,\widehat P_{N(s)})\le\varepsilon_s\bigr\},
\quad
\varepsilon_s=\varepsilon_0/\sqrt{N(s)},$ where $N(s)$ is the number of times state $s$ is visited. We then optimize the \emph{worst‑case} CVaR within this set. We show that   
for any $\delta\in(0,1)$,
$
\Pr\!\Bigl[\,
\sup_{\widetilde P\in\mathcal{P}_{\varepsilon_s}}\!
\text{CVaR}_{\alpha}^{\widetilde P}(C_H)\le\boldsymbol{\tau}\Bigr]\ge1-\delta,$
so W‑MCTS remains tail‑safe under distributional shifts. Further,   
The algorithm satisfies
$\text{Regret}(T)=\tilde{\mathcal O}(L_C\varepsilon_0\sqrt{T}),
$  where $L_C$ is the Lipschitz constant dependent on the risk-measure. This tightens as $\varepsilon_s\!\downarrow$with more node visits, preserving
CVaR safety satisfaction while maintaining sub‑linear regret.

In summary, this paper advances state-of-the-art C-MCTS solutions from two aspects: (i) enabling MCTS with PAC tail-safety guarantees by CVaR-MCTS; and (ii) further improving the guarantee under finite
samples by W-MCTS with first-order Wasserstein
ambiguity sets. The results lay a solid foundation for safely deploying MCTS-based agents in high-stake scenarios. Extensive experiments across various benchmarks validate that our proposed CVaR-MCTS and W-MCTS effectively control tail-risk, consistently outperforming existing risk-sensitive methods in terms of both safety and efficiency.

%% file: 03Background.tex
\section{Background}

\noindent \textbf{Safe RL.}
Traditional reinforcement learning methods, such as Q-learning~\cite{watkins1992q}, DQN~\cite{mnih2013playing}, PPO~\cite{schulman2017proximal} and so on ~\cite{guo2023advantage, qiao2024br,zhang2024distributed,zhang2025network,gao2024cooperative,fang2023implementing,fang2024learning}, typically focus on reward maximization while ignoring safety factors present in real-world scenarios. While efficient risk-sensitive RLs \cite{wang2023near,wang2020reward,fei2020risk} have been proposed, they again focus on maximizing the tail-risk measures of the reward ignoring the safety constraints. 
To address this issue, the Constrained Markov Decision Process (CMDP) was proposed and has become the foundation for many SafeRL algorithms. Based on this framework, methods such as Constrained Policy Optimization (CPO)~\cite{achiam2017constrained}, Lyapunov-based RL (LRL)~\cite{chow2018lyapunov}, Primal-dual-based approaches \cite{ding2020natural,ghosh2022provably,efroni2020exploration,liu2021learning,wei2021provably}, and LP-based approaches \cite{efroni2020exploration} have been proposed.  Risk-constrained MDPs have also been considered \cite{chow2018risk,ghoshonline,ahmadi2021constrained,brazdil2020reinforcement,misra2023design}. Although these approaches have considered different safety constraints and risk measures, including CVaR~\cite{rockafellar2002conditional}, they mostly focus on value- or policy-based RL methods and do not directly apply to MCTS, which offers significant benefits in terms of sampling efficiency, online learning,  handling high-dimensional state spaces, and efficient exploration and exploitation. 

\textbf{Constrained-MCTS}
MCTS has been widely applied due to its superior performance in complex decision-making tasks such as chess and board games~\cite{silver2017mastering}. 
Constrained Monte Carlo Tree Search (C-MCTS) extends MCTS by incorporating constraint satisfaction during tree expansion, drawing on CMDP principles to improve safety. Prior methods—such as Multi-objective MCTS~\cite{wang2012multi}, POMDP-based MCTS~\cite{lee2018monte}, and Primal-Dual C-MCTS~\cite{parthasarathy2023c}—address constraints but typically focus on mean-risk or hard threshold formulations. These become unreliable under limited samples due to high variance or bias. In contrast, we propose a framework that provides finite-sample theoretical guarantees for tail-risk safety in MCTS. While~\cite{hayes2022montecarlotreesearch} explores risk-aware tree search, it lacks theoretical guarantees which we provide.

%% file: 04Solution.tex
\section{Preliminaries}
\textbf{CMDP}
The Constrained Markov Decision Process (CMDP) extends the classical MDP 5-tuple $<\mathcal{S}, \mathcal{A}, \mathcal{P}, r, \gamma>$ by additionally introducing $K$ cost functions and their corresponding thresholds to capture the trade-off between reward maximization and safety, resource, or risk control. It is represented as the 7-tuple:$<\mathcal{S}, \mathcal{A}, \mathcal{P}, r, \{c^{(k)}\}_{k=1}^{K}, \gamma, \boldsymbol{\tau}>$
where $\mathcal{S}$ denotes the state space, $\mathcal{A}$ denotes the action space, $P(s' \mid s, a)$ is the transition function, and $\gamma \in [0,1]$ is the discount factor. The function $r: \mathcal{S} \times \mathcal{A} \rightarrow \mathbb{R}$ defines the reward function. The function $c^{(k)}: \mathcal{S} \times \mathcal{A} \rightarrow \mathbb{R}_{\geq 0}$ specifies the k-th cost function,  and $0 \leq c^{(k)}(s_t, a_t) \leq 1$, for $k = 1, \ldots, K.$ The vector $\boldsymbol{\tau} = (\tau_1, \tau_2, \ldots, \tau_K)^{T}$ gives the maximum allowable expected threshold for each cost. $c=\{c^{(k)}\}_{k=1}^K$ is the cost vector. 
Under policy $\pi$, within a planning horizon $H$, the undiscounted cumulative cost vector is defined as $C_H = \sum_{t=0}^{H-1} c(s_t, a_t).$  CMDP is then defined as $\max_{\pi}J_r^{\pi}\quad \text{s.t } \mathbb{E}_{\pi}[C_H]\leq \boldsymbol{\tau}$. Here, $J_r^{\pi}$ is the average expected reward, $\mathbb{E}_{\pi}[\mathbb{E}_{s\sim \rho} V_{\pi}(s)]$; $V^{\pi}(s)=\mathbb{E}_{\pi}[\sum_{t=1}^{\infty}\gamma^{t-1}r(s_t,a_t)|s_1=s]$, and $\rho$ is the initial state distribution. CMDP ensures that only the average cost is below a certain threshold. However, it does not provide guarantee on tail-distribution.

\noindent \textbf{CVaR}
To explicitly characterize the potential tail-risks that may arise in the decision process, Conditional Value-at-Risk (CVaR) has been proposed as a tail-risk measure. For a random variable $Z$ and a given confidence level $\alpha \in (0,1)$, the Value-at-Risk (VaR) is first defined as follows:

\begin{equation}
\begin{aligned}
    \text{VaR}_{\alpha}(Z) = \inf\{z | \mathbb{P}(Z\leq z)\geq \alpha\}
\end{aligned}
\end{equation}
Based on this, CVaR is further defined as follows:
\begin{equation}
\begin{aligned}
    \text{CVaR}_{\alpha}(Z) = \frac{1}{1-\alpha} \int_{\alpha}^{1}\text{VaR}_{u}(Z) du.
\end{aligned}
\end{equation}
VaR provides the quantile corresponding to the worst $(1-\alpha)\%$ scenarios, while CVaR further takes the expectation over this range, offering a continuous characterization of tail losses and making it more amenable to optimization.
If the risk dimension is $K$, the CVaR of $C_H$ is computed component-wise as $
\text{CVaR}_{\alpha}(C_H) = \bigl[\text{CVaR}_{\alpha}^{(1)}, \ldots, \text{CVaR}_{\alpha}^{(K)}\bigr]^{\top},$ and a vector constraint is imposed: $\text{CVaR}_{\alpha}(C_H)\leq \boldsymbol{\tau}.$

\subsection{CVaR-MCTS}

To effectively control the tail-risk constraint cost in the decision process, we integrate the CVaR risk measure with the Monte Carlo Tree Search (MCTS) algorithm and propose a CVaR- MCTS (CVaR-MCTS) algorithm framework. Due to page limitations, proofs of all theorems are presented in the appendix. 

First, we explicitly pose the problem as follows:  the goal is to maximize the expected return while maintaining the $\text{CVaR}_{\alpha}$ of the cost below the threshold $\boldsymbol{\tau}$ at state $s$ over the planning horizon $H$.
\begin{align}\label{eq:cvar-mcts}
\max_{\pi} J_r^{\pi}, \quad \text{s.t }
\text{CVaR}_{\alpha}(C_H(\pi)) \leq \boldsymbol{\tau}.
\end{align}
where $J_r^{\pi}=[\sum_{t=1}^H\gamma^{t-1}r_t(s_t,a_t)|s_1=s]$.

In order to solve this, we introduce a non-negative Lagrange multiplier $\boldsymbol{\lambda} \geq 0$ and construct the Lagrangian function as follows.
\begin{equation}
\begin{aligned}
    \mathcal{L}(\pi,\lambda) = J_{r}(\pi) - \sum_{k=1}^{K}\lambda_{k}(\text{CVaR}_{\alpha}^{(k)}\left(C_{H}(\pi))-\boldsymbol{\tau}_{k}\right)
\end{aligned}
\end{equation}
The above formulation indicates that, during the node selection phase, the algorithm not only takes into account the reward estimates and the exploration term, but also explicitly incorporates the CVaR tail risk constraints, thereby achieving a balance between reward and risk control. 

 For a fixed policy $\pi$, we derive the dual function with respect to $\boldsymbol{\lambda}$.
\begin{equation}
\begin{aligned}
    D(\boldsymbol{\lambda}) = \max_{\pi} \mathcal{L}(\pi,\boldsymbol{\lambda})
\end{aligned}
\end{equation}
Since for each policy $\pi$, the summation over $\boldsymbol{\lambda}$ is linear, $\boldsymbol{\lambda}$ is a convex function of $\mathcal{L}$. Moreover, the dual function $D(\boldsymbol{\lambda})$ is obtained by taking the pointwise maximum over all possible $\pi$, so $D(\boldsymbol{\lambda})$ must be a convex function. Therefore, the dual optimal solution is:
\begin{equation}
\begin{aligned}
    \boldsymbol{\lambda}^{*} = \arg\min_{\boldsymbol{\lambda\geq 0}} D(\boldsymbol{\lambda})
\end{aligned}
\end{equation}

Its gradient can be expressed as:
\begin{equation}
\begin{aligned}
    g(\boldsymbol{\lambda}) = \nabla D(\boldsymbol{\lambda}) = -(\text{CVaR}_{\alpha}(C_H(\pi^{*}(\lambda))) - \boldsymbol{\tau})
\end{aligned}
\end{equation}

where $\pi^{*}(\lambda)$ denotes the optimal policy that maximizes $\mathcal{L}$ given $\boldsymbol{\lambda}$.

Therefore, a dual update based on stochastic gradient ascent is naturally feasible.
At each internal node $s$ of the search tree, we design the following UCT-style selection rule

\begin{equation}
\label{eqn:cvar_ucb}
\begin{aligned}
U(s,a) = Q(s,a) + \beta_{R}\sqrt{\frac{\ln N(s)}{1+N(s,a)}} - \boldsymbol{\lambda}^{\top}_{s}\left(\widehat{\text{CVaR}_{\alpha}}(s,a) + \beta_{C}\sqrt{\frac{\ln N(s)}{1+ N(s,a)}}\textbf{1}- \textbf{B}_s\right)
\end{aligned}
\end{equation}
Here, $\widehat{\text{CVaR}_{\alpha}}(s,a)$ is the empirical tail-risk estimate obtained from the current set of roll-outs, while the factors premultiplied by $\beta_R$ and $\beta_C$ serve as optimism (upper-confidence) bonuses for the reward and cost terms, respectively. To guarantee feasibility with high probability, we enforce the constraint conservatively so that violations remain unlikely. Recall from~\eqref{eq:cvar-mcts} that the initial cost budget is $\tau$; as the episode progresses and costs accrue, the residual budget $\mathbf{B}_s$ is updated to reflect the remaining allowance at state $s$. {\em Note that the overall policy now depends on the available budget $\mathbf{B}_s$ which is consistent with approaches for risk-constrained MDP.}

To enable $\boldsymbol{\lambda}$ and $\mathbf{B}$ to gradually adapt to the true risk distribution during the search process, we employ online gradient ascent/descent dual updates:

\begin{equation}
\begin{aligned}
    \boldsymbol{\lambda}_{t+1} = \left[\boldsymbol{\lambda}_t + \eta_t g(\boldsymbol{\lambda})\right]_{+}
\end{aligned}
\end{equation}

In practical algorithms, we cannot compute $g(\boldsymbol{\lambda})$ exactly; instead, we construct a stochastic gradient using the CVaR estimate $\widehat{\text{CVaR}}_{\alpha}$ from a single Monte Carlo trajectory. At the same time, we update the existing budget $\mathbf{B}$. Accordingly, the updates can be written as:

\begin{equation*}
\begin{aligned}
    \boldsymbol{\lambda_{s}} \leftarrow \left[\boldsymbol{\lambda}_{s}+\eta_t (\widehat{\text{CVaR}}_{\alpha}- \textbf{B}_{s})\right]_{+}, \textbf{B}_{s}\leftarrow \textbf{B}_{s}-\widehat{\text{CVaR}}_{\alpha}(s,a)
\end{aligned}
\end{equation*}

The first rule ensures that when the estimated CVaR exceeds the current threshold $\mathbf{B}_s$, $\boldsymbol{\lambda}_{s}$ increases, thereby penalizing high-risk actions to a greater extent in subsequent selections; conversely, it decreases if the threshold is not exceeded. The second rule adjusts the budget $\mathbf{B}_s$ to dynamically track the actual level of sample risk and avoid being overly conservative or excessively relaxed.

\begin{algorithm} 
	\caption{CVaR-MCTS} 
	\label{alg3} 
	\begin{algorithmic}
		\STATE \textbf{Init:} $\boldsymbol{\lambda}_{\text{root}}\leftarrow \textbf{0}, \textbf{B}_{\text{root}}\leftarrow \boldsymbol{\tau}$
            \FOR{$t = 1,2,..., T$}
                \STATE \textbf{Select:} Follow the action with the largest $U(s,a)$ from the root node.
                \STATE \textbf{Expand:} Add the first-visited leaf node to the tree. 
                \STATE \textbf{Rollout:} Simulate according to policy $\pi$ and obtain the reward and cost.
                \STATE \textbf{Backprop:} update $Q,\widehat{\text{CVaR}},\boldsymbol{\lambda},\textbf{B}$
            \ENDFOR
	\end{algorithmic} 
\end{algorithm}

To ensure safety, we need to provide formal guarantees on the violation and convergence of our proposed CVaR-MCTS. Next, we present our theoretical analysis to establish such guarantees.

\begin{Theorem}[PAC-CVaR safety Guarantee]
\label{thm:cvar_safety}
For the one-step cost $c \in [0,1]$, $\beta_{C} = \sqrt{2}$, and empirical estimate satisfying $\widehat{\text{CVar}_{\alpha}}\leq \tau$,  we have with probability at least $1-\delta$:
\begin{equation*}
\begin{aligned}
    \text{CVaR}_{\alpha}(C_{H})\leq \boldsymbol{\tau}+\epsilon\mathbf{1} , \forall N \;\ge\;
\frac{2\beta_C^{2}}{(1-\alpha)^{2}}
\cdot
\frac{\ln(2K/\delta)}{\varepsilon^{2}}.
\end{aligned}
\end{equation*}
\end{Theorem}
The theorem shows that once each internal node has been visited at least
$N$
times—where \(N\) depends on the target violation probability \(\delta\) and the
constraint dimension \(K\), we guarantee with high probability $(1-\delta)$ that the \emph{true} CVaR satisfies
the safety constraint \(\boldsymbol{\tau}\) up to $\epsilon$, as long as the empirical CVaR satisfies the constraint \(\boldsymbol{\tau}\). 
We note that \(N\) increases with \(\alpha\): a larger \(\alpha\) (i.e.,
a more risk-averse objective further mitigating the tail) requires
more samples. The lower bound shows that the required \(N\) scales
\emph{logarithmically} with the constraint dimension~\(K\), mirroring
standard confidence-bound analyses.  Hence the theorem offers a
PAC-style safety guarantee for tail-risk, and enables selection of proper \(N\) to achieve any desired accuracy \(\epsilon\). We also that $N=O(1/\epsilon^2)$ for minimizing the $\epsilon$-gap. Next, for squire-summable learning rate $\eta_t$, we analyze convergence of CVaR-MCTS.

\begin{Theorem}[Convergence Guarantee]
\label{thm:param_convergence}
For cost $c \in [0,1]$ and squire-summable learning rate $\eta_t$ satisfying
$\sum_{t}\eta_t = +\infty$ and $\sum_{t}\eta_t^2 < +\infty$, the multiplier sequence $\{\boldsymbol{\lambda}_{s,t}\}$ almost surely converges for any internal nodes $s$ to optimal dual solution $\boldsymbol{\lambda}_{s}^\star$, and satisfies$
||\boldsymbol{\lambda}_{s,t} - \boldsymbol{\lambda}_{s}^{*}||_{1} = \mathcal{O}\left(\frac{1}{\sqrt{t}}\right).$
\end{Theorem}

This theorem shows the convergence of multipliers $\{\boldsymbol{\lambda}_{s,t}\}$, which directly implies the convergence of MCTS following standard analysis. The conclusion is consistent with classical stochastic approximation theory: by setting a decaying step size that is not summable but square-summable,
the dual variable updates are guaranteed to converge at a rate of $\mathcal{O}\left(\frac{1}{\sqrt{t}}\right)$, without requiring strong Lipschitz condition.
This naturally holds in MCTS, since the number of updates $t \to \infty$ at each node increases with its visit count, thereby eliminating the need of the condition.

\begin{Definition}[Constrained Cumulative Regret $\mathcal{R}_{T,\tau}$] \label{def:regret}
Let $R_t$ be the return of the $t$-th rollout and
$J^\star=\max_{\pi:\,\mathrm{CVaR}_{\alpha}\le\boldsymbol{\tau}}\!
\mathbb{E}_{\pi}[R]$ under safety constraint $\tau$.
We define the following constrained cumulative regret
$\mathcal{R}_{T,\tau}=\sum_{t=1}^{T}(J^\star-R_t)$, quantifying the gap between CVaR-MCTS and an ideal optimal solution.
\end{Definition}

\begin{Theorem}[Regret of CVaR-MCTS]
\label{thm:regret_cvar}
The constrained cumulative regret of CVaR-MCTS satisfy
\begin{equation*}
\begin{aligned}
\mathcal{R}_{T,\tau} \leq c\sqrt{T\ln T}
\end{aligned}
\end{equation*}
where 
constant $c$ is $\mathcal{O}(\beta_R + \beta_C K + H)$.
\end{Theorem}

The proof is non-trivial, since the constrained cumulative regret is defined only with respect to the expected return (under safety constraint), while action selections in MCTS follow $U(s,a)$ in Eq.(\ref{eqn:cvar_ucb}), which depends on both return estimate and empirical CVaR estimate. We need to apply Hoeffding inequality to both of these estimates and combine the results to bound the constrained cumulative regret. Interestingly, the regret upper bound is of the same order as that of standard UCT-style methods, indicating that the exploration–exploitation trade-off is not significantly compromised by the introduction of the CVaR penalty. 
In the constant $c$, $\beta_R$ governs the strength of the reward bonus, $\beta_C$ accounts for the contribution of the risk-related confidence term to the regret, $K$ is the number of constraints, and the tree depth $H$ reflects the planning cost per rollout.
CVaR-MCTS ensures favorable convergence properties while ensuring risk control.

\subsection{W-MCTS}
MCTS typically learns a dynamics function~\cite{silver2016mastering} to produce next hidden states and rewards during tree search. Since we empirically roll out the tree to estimate $\widehat{\text{CVaR}_{\alpha}}$ and update the CVaR penalty in CVaR-MCTS. The epistemic uncertainty of the learned dynamics function (due to finite visits $N(s)$ to each state $s$) can affect the empirical estimate $\widehat{\text{CVaR}_{\alpha}}$, potentially leading to a new source of bias in CVaR-MCTS. To mitigate the bias in empirical CVaR estimates caused by model uncertainty in sparsely visited states, we introduce Wasserstein MCTS (W-MCTS), a robust extension of CVaR-MCTS that leverages state-dependent Wasserstein ambiguity sets. Unlike CVaR-MCTS, which assumes uniform model accuracy across the tree, W-MCTS adaptively scales the radius of each uncertainty set based on the local visitation count 
$N(s)$, providing tighter estimates in well-sampled regions and conservative adjustments in uncertain ones. This state-aware robustness improves tail-risk estimation under epistemic uncertainty, enabling safer and more reliable planning. 
Our idea is to introduce a first-order Wasserstein ambiguity set $\mathcal{P}_{\varepsilon_{s}}(s,a)$ with radius $\varepsilon_{s}$ to model the epistemic uncertainty in the processing of learning the dynamics function. The radius $\varepsilon_{s}$ decreases over time, as states are revisited and new samples are obtained. We factor this ambiguity set into estimating $\widehat{\text{CVaR}_{\alpha}}$ as in robust MDP literature \cite{xu2023improved}.Formally, we define ambiguity set $\mathcal{B}_{W}(\hat{P},\varepsilon_{s}) = \{\tilde{P}|W_{1}(\tilde{P},\hat{P})\leq \varepsilon_{s}\}$, 
where $W_{1}(\tilde{P}, \hat{P})$ denotes the first-order Wasserstein distance between the true dynamics $\tilde{P}$ and the empirically-learned dynamics function $\hat{P}$. Employing the Talgrand’s inequality, we can recast this into a robustness upper bound for the risk estimation:
\begin{equation*}
\begin{aligned}
    \sup_{P\in \mathcal{P}_{\varepsilon_{s}}} \text{CVaR}_{\alpha}^{P}(s,a)\leq \widehat{\text{CVaR}}_{\alpha}(s,a) + L_{C}\varepsilon_{s}
\end{aligned}
\end{equation*}
where Lipschitz constant $L_{C}$ depends on the support range of the returns and quantifies the sensitivity of the risk measure to perturbation. From the convergence rate given by Theorem~\ref{thm:param_convergence}, it is easy to show that we must maintain $\mathbb{E}\bigl[W_1(P, \hat{P}_N)\bigr] = \mathcal{O}\bigl(N^{-1/2}\bigr)$. Therefore, we set $\varepsilon_s = \varepsilon_{0} / \sqrt{N(s)}$ to maintain the same convergence rate of W-MCTS, for some constant $\epsilon_0$. Although, in our algorithm we need to know $\epsilon_0$, we do not use this information for empirical results. 

Next, we consider a new robust cost estimate of the tail risk as follows:
\begin{equation*}
\begin{aligned}
    C_{\text{worst}}(s,a) = \widehat{\text{CVaR}}_{\alpha}(s,a) + L_{C}\varepsilon_{s}, \varepsilon_{s} = \varepsilon_{0}/\sqrt{N(s)}.
\end{aligned}
\end{equation*}
It can be observed that as the number of visits $N(s)$ to the state node $s$ increases, the radius $\varepsilon_{s}$ of the ambiguity set will automatically shrink, and the robustness upper bound of the risk estimation will consequently converge to the true value in the process.
The constant $\varepsilon_0$ can be estimated based on historical data or domain knowledge—for example, by measuring the distributional discrepancy between simulated and real-world trajectories.

By substituting the above robust cost estimate $C_{\text{worst}}(s,a)$ into the previously described CVaR-MCTS selection criterion, we obtain the W-MCTS algorithm with enhanced robustness.
\begin{equation}
\begin{aligned}
U(s,a) =& Q(s,a) + \beta_{R}\sqrt{\frac{\ln N(s)}{1+N(s,a)}} - \boldsymbol{\lambda}_{s}^{\top}\left(C_{\text{worst}}(s,a) + \beta_{C}\sqrt{\frac{\ln N(s)}{1+ N(s,a)}}\textbf{1}- \mathbf{B}_s\right)
\end{aligned}
\end{equation}
Our evaluation later shows that W-MCTS can significantly improve the performance of tree search under constraints. We next provide probabilistic safety guarantee and analyze convergence of W-MCTS.

\begin{Theorem}[W-PAC safety]
\label{thm:w_safety}
For 
Lipschitz constant $L_C$ and any $\tilde{P} \in \mathcal{B}_{W}(\hat{P}, \varepsilon_{s})$ in the ambiguity set if $\widehat{\text{CVaR}}_{\alpha}+L_C\epsilon_s\leq \tau$, then, we have 
\begin{equation}
\begin{aligned}
    \Pr_{\widetilde{P}}[\text{CVaR}_{\alpha}(C_{H}) \leq \tau] \geq 1- \frac{2}{T^2}-2\exp(-2\varepsilon_0^2)
\end{aligned}
\end{equation}
\end{Theorem}
This indicates that W-MCTS still satisfies the risk constraints with high probability as long as the $C_{\text{worst}}\leq \tau$. Note that the bound inherently depends on the initial model error $\epsilon_0$. Also, note that we remove the dependency on $N$ unlike in Theorem~\ref{thm:cvar_safety} using the Wasserstein model. We next show its regret bound.
\begin{Theorem}[W-MCTS Robust Regret]
\label{thm:regret}
The regret of W-MCTS satisfies:
\begin{equation}
\begin{aligned}
    \mathcal{R}_{T,\tau} \leq c_1\sqrt{T\ln T} + c_2 L_C \varepsilon_0 \sqrt{T}
\end{aligned}
\end{equation}
\end{Theorem}
Compared to the regret upper bound of CVaR-MCTS, W-MCTS introduces an additional term $c_2 L_C \varepsilon_0$, which arises from introducing the ambiguity set (which makes W-MCTS more robust but introduces an additional regret gap).
As we will demonstrate through numerical examples later, for $\varepsilon_0$ and as the number of visits increases, the impact of this term gradually diminishes, and the algorithm still retains favorable sublinear regret performance.

%% file: 05Experment.tex
\section{Numerical Evaluation}

We compare CVaR-MCTS with a number of baselines including Vanilla-MCTS, C-MCTS, Risk-Averse MCTS, W-MCTS (our approach), TRPO-Lagrange, CPPO, in both Grid-World-Hazard and more complex traffic simulation tasks.  
Experiments are performed on a server with an AMD EPYC 7513 32-Core Processor CPU and an NVIDIA RTX A6000 GPU. Further details are provided in the Appendix.

\subsection{Mechanism Validation and Analysis}

We first select the 
Grid-World-Hazard environment to validate our approaches compared to Vanilla-MCTS, because its discrete structure and analytically tractable optimal policy let us isolate the effect of tail-risk control in a clean setting, while sparse, high-impact hazard cells faithfully mimic rare but catastrophic safety events found in real-world tasks.  An agent starts from the upper-left corner and need to move to the the goal at the lower-right corner (Fig.~\ref{fig:grid_layout}), across a random map with red hazard cells incurring a one-time cost of $c_t=1$ (while all other cells have zero cost). We set the risk threshold to $\tau = 0.2H$, the discount factor to $\gamma = 0.99$, and the planning horizon to $H = 25$.  To  illustrate the search strategy and risk distribution, we visualize the following results:

\begin{figure*}[h]
\centering
{\fontsize{10}{12}
  \subfigure[Grid layout]{\includegraphics[width=0.24\textwidth]{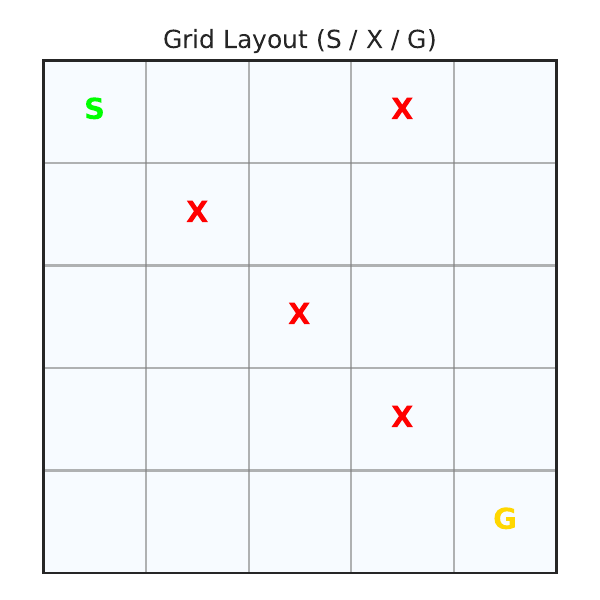}
  \label{fig:grid_layout}
  }\hfill
  \subfigure[Vanilla‑MCTS]{\includegraphics[width=0.24\textwidth]{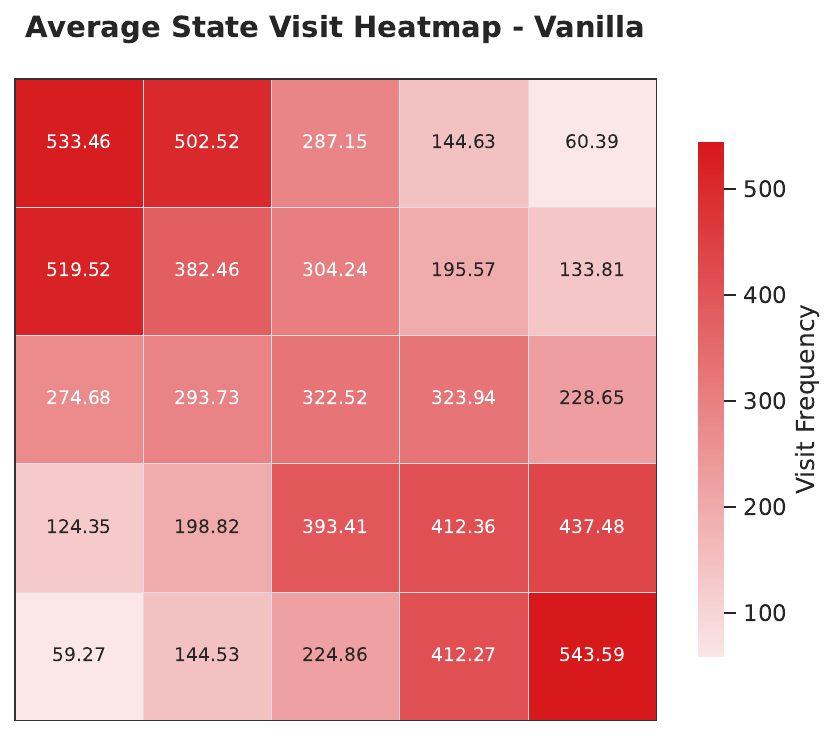}
  \label{fig:heatmap_Vanilla}
  }\hfill
  \subfigure[CVaR‑MCTS]{\includegraphics[width=0.24\textwidth]{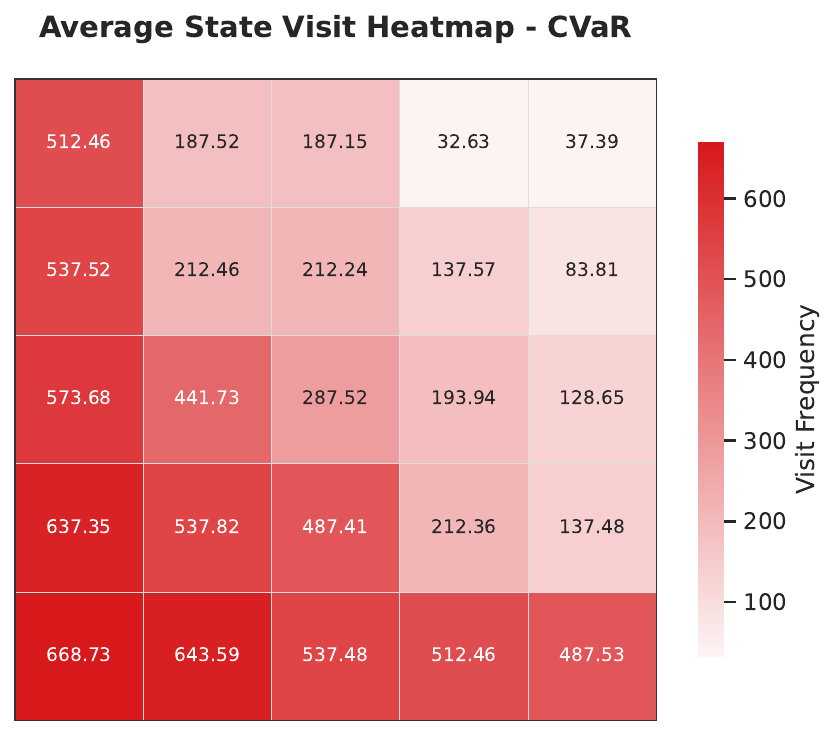}
  \label{fig:heatmap_CVaR}
  }\hfill
  \subfigure[W‑MCTS]{\includegraphics[width=0.24\textwidth]{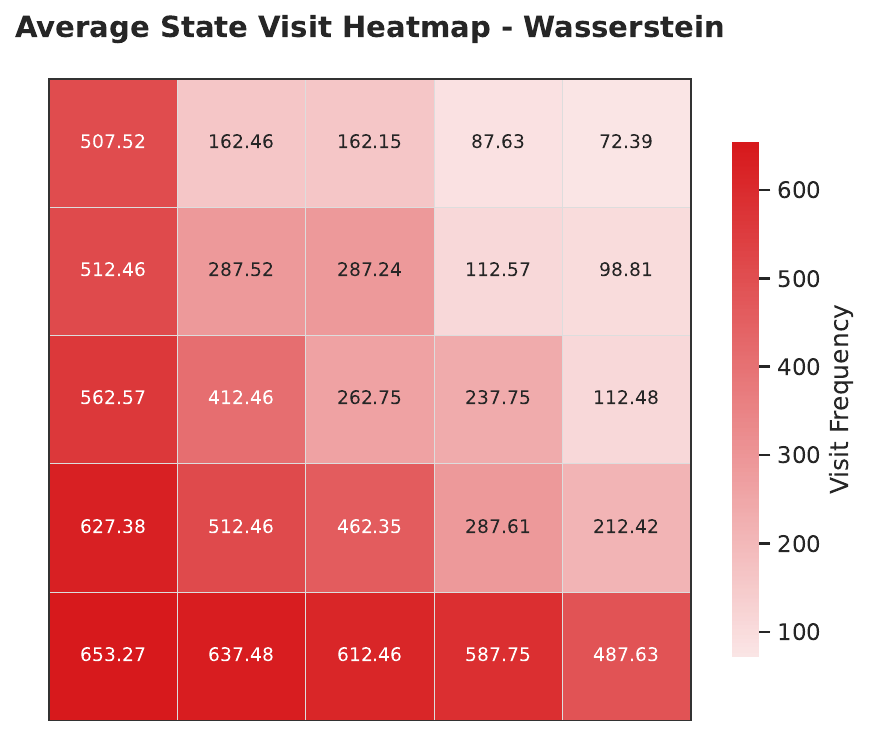}
  \label{fig:heatmap_Wasserstein}
  }\hfill
  \caption{State visitation heatmaps in Grid‑World‑Hazard. Figure~\ref{fig:grid_layout}, Grid layout. Start (S), goal (G), and hazard (“X”) cells. Figure~\ref{fig:heatmap_Vanilla}, Vanilla‑MCTS. Cells most frequently visited by the standard UCT agent often traverse hazards. Figure~\ref{fig:heatmap_CVaR}, CVaR‑MCTS. The tail‑risk‑aware search steers clear of high‑cost cells, favoring safer routes. Figure~\ref{fig:heatmap_Wasserstein}, W‑MCTS. The distributionally‑robust variant further reduces visits to hazardous cells under uncertainty.
  }
  \label{fig:heatmap}
} 
\end{figure*}

\begin{figure*}[h]
\centering
{\fontsize{10}{12}
  \subfigure[$\text{CVaR}_{0.9}$ convergence]{\includegraphics[width=0.32\textwidth]{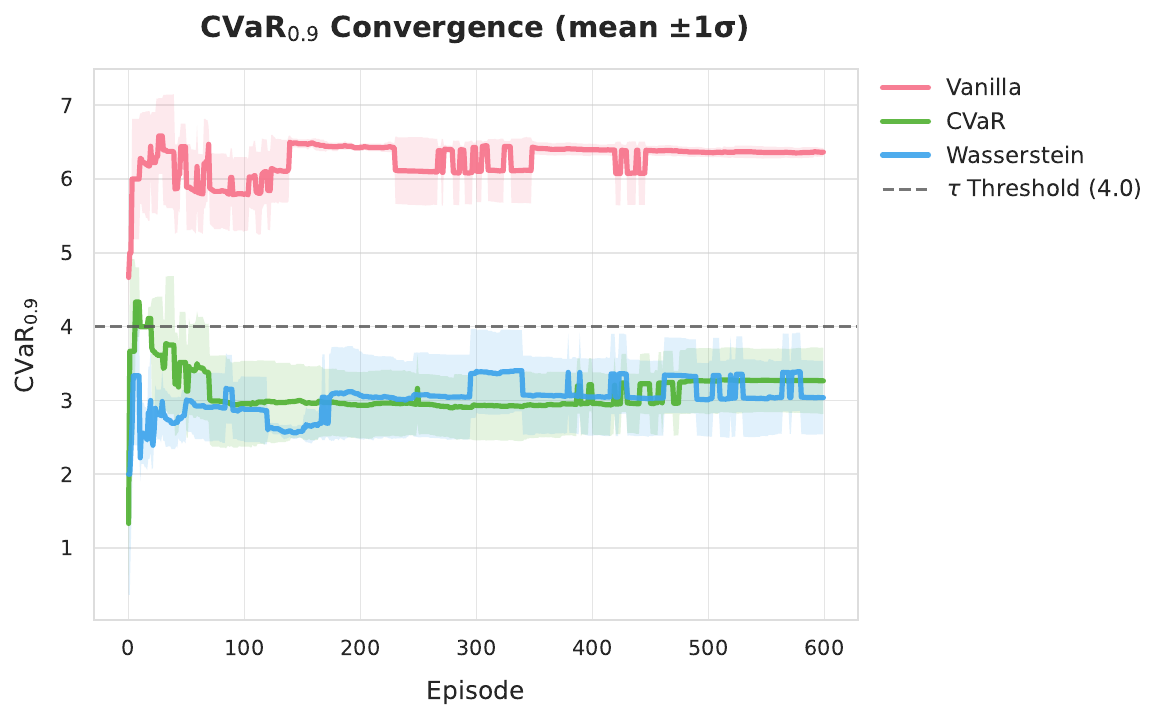}
  \label{fig:cvar_09}
  }\hfill
  \subfigure[Path risk distribution]{\includegraphics[width=0.32\textwidth]{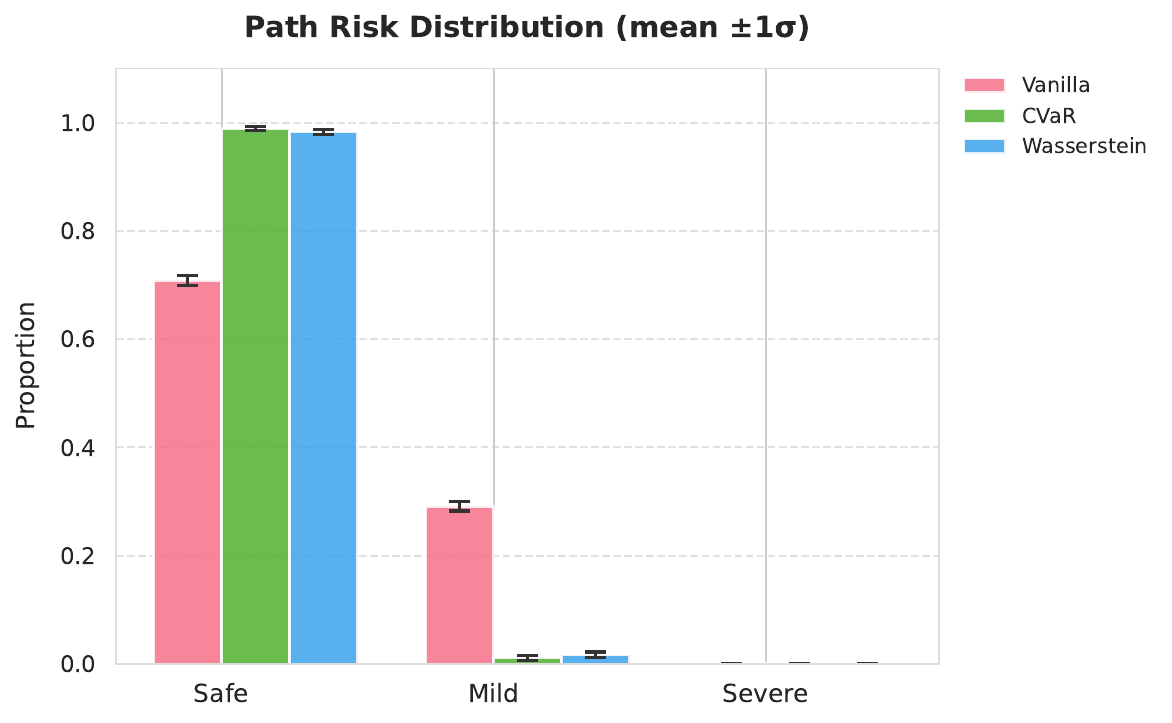}
  \label{fig:safe_path}
  }\hfill
  \subfigure[Tail loss density]{\includegraphics[width=0.32\textwidth]{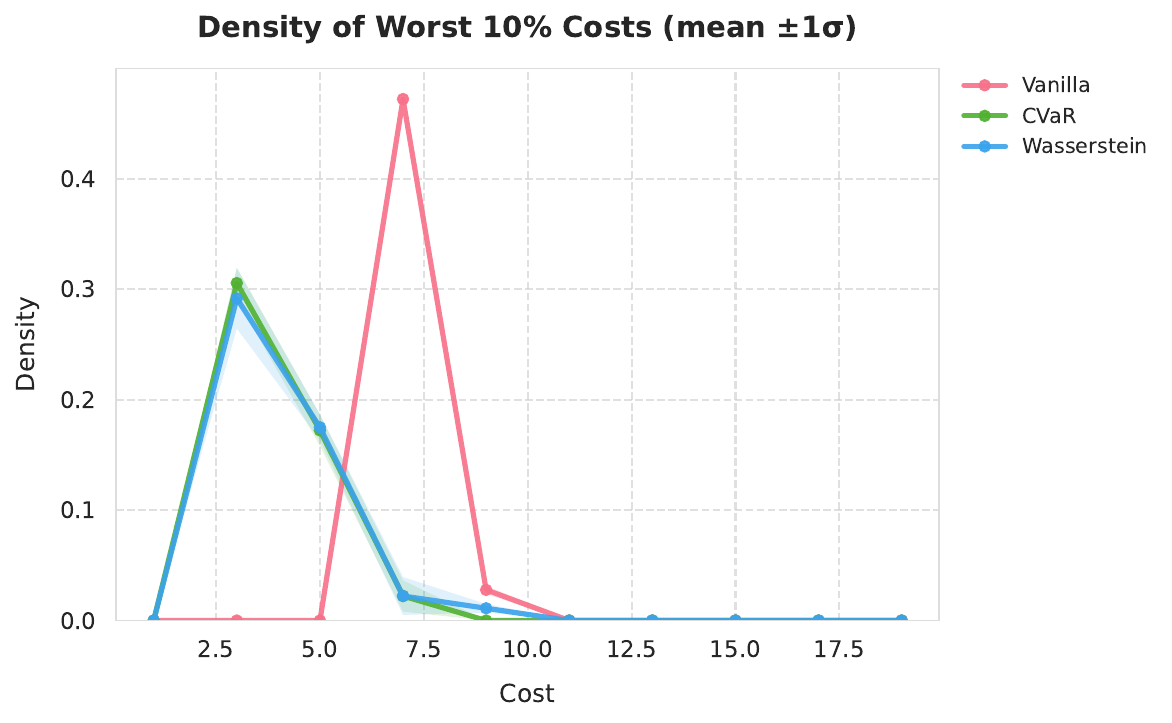}
  \label{fig:cost_distribution}
  }\hfill
  \caption{Risk and cost statistics over episodes. Figure~\ref{fig:cvar_09}, $\text{CVaR}_{0.9}$ convergence mean $\pm 1$ std error of the 90\%-CVaR of cumulative cost versus episode count, with the safety threshold shown. Figure~\ref{fig:safe_path}, Path risk distribution. Proportion of trajectories classified as "safe" (cost $\leq \tau$), "mildly violating," or "severely violating" under each algorithm (mean $\pm 1$ std). Figure~\ref{fig:cost_distribution}, Tail loss density. Kernel‑density estimate of the top 10\% cost outcomes, illustrating how CVaR‑MCTS and W‑MCTS achieve lower extreme losses than Vanilla‑MCTS.
  }
  \label{fig:mpe}
  \vspace{-0.2in}
} 
\end{figure*}

(i) \textbf{State visitation heatmap (Figure~\ref{fig:heatmap})} — the average visitation frequency of each grid cell after 1,000 trials of different algorithms, where darker colors indicate higher frequencies. It is observable that Vanilla-MCTS tends to follow the diagonal path with the highest sampling probability but frequently steps on hazards, whereas  our approaches CVaR-MCTS and W-MCTS 
prioritize paths with minimal cost.

(ii) \textbf{$\mathrm{CVaR}_{0.9}$ convergence curves (Figure~\ref{fig:cvar_09})} — the vertical axis shows the cumulative cost $\mathrm{CVaR}_{0.9}$, and the horizontal axis shows the number of episodes. CVaR-MCTS and W-MCTS rapidly reduce the CVaR below the threshold within 2,000 episodes, with W-MCTS further tightening the confidence bands, demonstrating the finite-sample robustness of Theorems~\ref{thm:cvar_safety} and~\ref{thm:w_safety}. In contrast, Vanilla-MCTS fails to reduce the CVaR below threshold $\tau$.

(iii) \textbf{Path distribution bar chart (Figure~\ref{fig:safe_path})} — this chart reports the proportion of trajectories with different risk levels (safe, mildly violating, severely violating). The proportion of safe trajectories under W-MCTS and CVaR-MCTS is significantly higher than that under Vanilla-MCTS.

(iv) \textbf{Tail loss density estimation (Figure~\ref{fig:cost_distribution})} — the kernel density curves illustrate the cost distribution on the worst 10\% of trajectories for each algorithm. It can be observed that the cost distributions experienced by CVaR-MCTS and W-MCTS are clearly lower than that of Vanilla-MCTS.

This experiment verifies the core capability of the two risk-constrained MCTS algorithms: maintaining performance while consistently satisfying the $\tau$-level tail-risk requirement. Among them, W-MCTS achieves the most concentrated tail-loss distribution under limited samples, establishing a reliable baseline for subsequent experiments in complex traffic scenarios.

\subsection{Complex Safety Scenarios: Overall Performance Comparison}
We compare the performance of different algorithms (Vanilla-MCTS, C-MCTS, Risk-Averse MCTS, CVaR-MCTS, W-MCTS, TRPO-Lagrange, CPPO) in several complex traffic simulation environments (Highway, Intersection, Racetrack, Roundabout).These tasks are implemented based on the \textit{Highway-env} simulator~\cite{highway-env}.
Highway-env is a two-dimensional multi-lane highway driving simulation environment (with 4 lanes by default and kinematics-based observations). All background vehicles follow the IDM model and drive at speeds ranging from 20 to 40 m/s, changing lanes when necessary. The agent (ego vehicle) operates with a discrete action set {\texttt{LANE\_LEFT}, \texttt{LANE\_RIGHT}, \texttt{FASTER}, \texttt{SLOWER}, \texttt{IDLE}}, and continuous control is also supported.
All experiments adopt a unified setting with a planning horizon of $H = 20$ steps and a tail-risk budget of $\tau = 0.2H$.
The evaluation metrics include average reward (Reward), risk-sensitive metric CVaR${0.9}$, average cost (Cost), and steps needed to reach 95\% convergence (Steps${95\%}$).

In this experiment, we normalize both the cost and reward at each step to a unified range of [0,1], to support a direct comparison between risk and reward on the same scale. Specifically, any "crash" is treated as the most severe cost (1.0), while other penalties (such as negative rewards from collisions or how far the vehicle goes off the lane) are linearly mapped based on their absolute values and summed up. The total cost is then clipped to [0,1]. Meanwhile, the immediate reward is made up of two parts: the ratio of current speed to maximum speed, and the environment’s “stay-in-lane” score. These are either directly added or weighted and then clipped to [0,1]. This method gives strong penalties for dangerous events while also providing detailed feedback on daily driving behavior, allowing risk and reward to be fairly balanced and smoothly integrated in later tree search and CVaR evaluation.

\begin{figure*}[h]
\centering
{\fontsize{10}{12}
  \subfigure[Highway CVaR]{\includegraphics[width=0.24\textwidth]{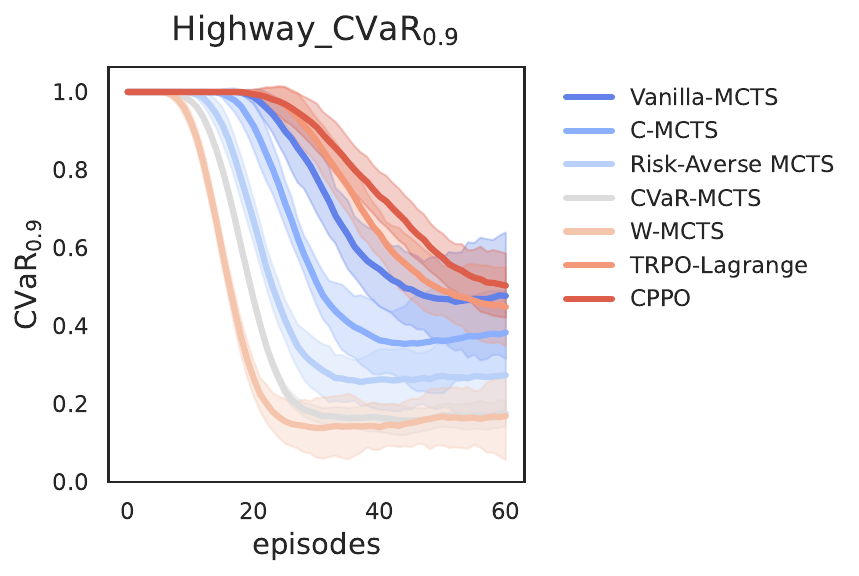}
  \label{fig:all_highway_cvar}
  }\hfill
  \subfigure[Intersection CVaR]{\includegraphics[width=0.24\textwidth]{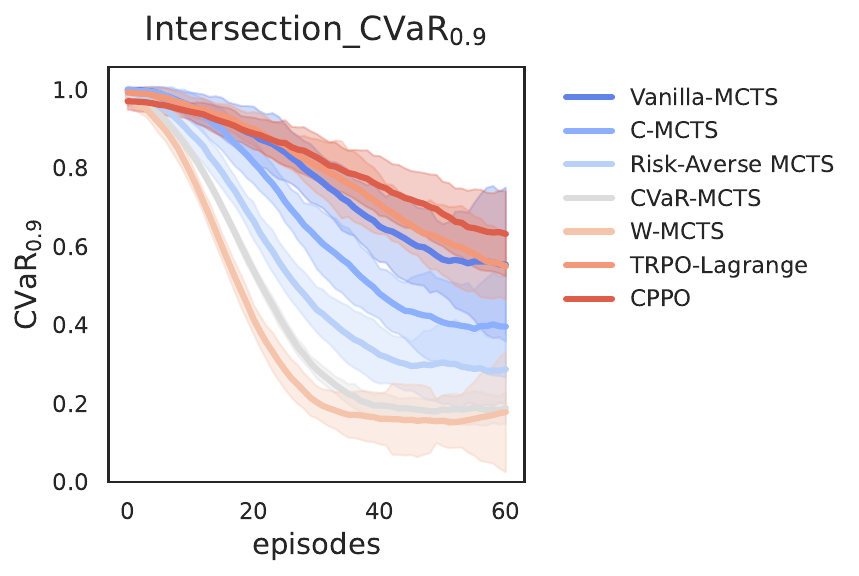}
  \label{fig:all_intersection_cvar}
  }\hfill
  \subfigure[Racetrack CVaR]{\includegraphics[width=0.24\textwidth]{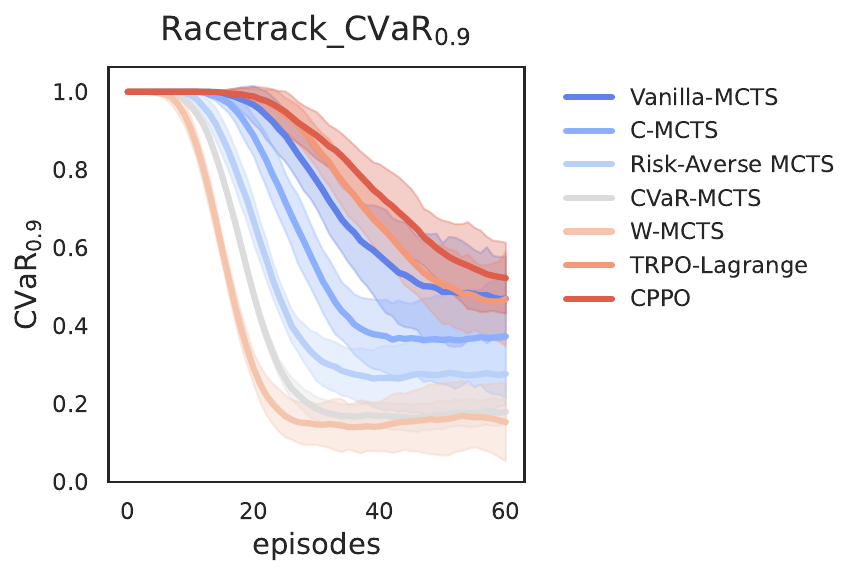}
  \label{fig:all_racetrack_cvar}
  }\hfill
  \subfigure[Roundabout CVaR]{\includegraphics[width=0.24\textwidth]{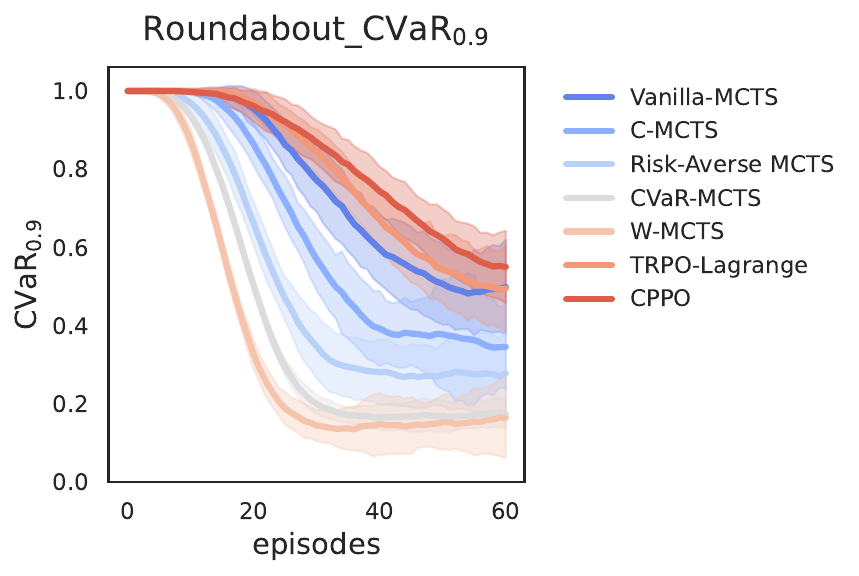}
  \label{fig:all_roundabout_cvar}
  }\hfill
  \caption{Performance trends for seven algorithms in four traffic environments. These show the 90\%‑CVaR of cumulative cost.  
  All curves are plotted as mean $\pm$ SEM across 30 independent runs.}
  \label{fig:all}
} 
\end{figure*}

\begin{table*}[h]
\centering
\label{tab:overall}
\begin{tabular}{llcccc}
\toprule\toprule
Algorithm & Env & Reward$\uparrow$ & CVaR$_{0.9}\downarrow$ & Cost$\downarrow$ & Steps$_{95\%}$$\downarrow$ \\
\midrule
C-MCTS & Highway & 0.535 ± 0.064 & 0.563 ± 0.023 & 0.448 ± 0.068 & 38.000000 \\
CPPO & Highway & 0.295 ± 0.059 & 0.788 ± 0.014 & 0.705 ± 0.050 & NaN \\
CVaR-MCTS & Highway & 0.763 ± 0.040 & 0.298 ± 0.016 & 0.221 ± 0.043 & 28.000000 \\
Risk-Averse MCTS & Highway & 0.666 ± 0.054 & 0.417 ± 0.021 & 0.313 ± 0.060 & 33.000000 \\
TRPO-Lagrange & Highway & 0.342 ± 0.060 & 0.739 ± 0.016 & 0.651 ± 0.054 & NaN \\
Vanilla-MCTS & Highway & 0.414 ± 0.076 & 0.695 ± 0.023 & 0.573 ± 0.074 & NaN \\
W-MCTS & Highway & \textbf{0.836 ± 0.045} & \textbf{0.236 ± 0.020} & \textbf{0.148 ± 0.047} & \textbf{23.0} \\ \midrule
C-MCTS & Roundabout & 0.533 ± 0.069 & 0.568 ± 0.026 & 0.452 ± 0.069 & 50.000000 \\
CPPO & Roundabout & 0.297 ± 0.062 & 0.789 ± 0.018 & 0.702 ± 0.051 & NaN \\
CVaR-MCTS & Roundabout & 0.758 ± 0.042 & 0.306 ± 0.017 & 0.228 ± 0.044 & 31.000000 \\
Risk-Averse MCTS & Roundabout & 0.662 ± 0.058 & 0.427 ± 0.023 & 0.321 ± 0.060 & 37.000000 \\
TRPO-Lagrange & Roundabout & 0.344 ± 0.065 & 0.750 ± 0.018 & 0.655 ± 0.056 & NaN \\
Vanilla-MCTS & Roundabout & 0.411 ± 0.083 & 0.710 ± 0.023 & 0.580 ± 0.077 & NaN \\
W-MCTS & Roundabout & \textbf{0.830 ± 0.046} & \textbf{0.240 ± 0.017} & \textbf{0.155 ± 0.046} & \textbf{27.0} \\ \midrule
C-MCTS & Intersection & 0.520 ± 0.076 & 0.593 ± 0.024 & 0.471 ± 0.070 & NaN \\
CPPO & Intersection & 0.307 ± 0.063 & 0.784 ± 0.020 & 0.694 ± 0.052 & NaN \\
CVaR-MCTS & Intersection & 0.738 ± 0.046 & 0.332 ± 0.016 & 0.253 ± 0.045 & 39.000000 \\
Risk-Averse MCTS & Intersection & 0.642 ± 0.064 & 0.453 ± 0.021 & 0.348 ± 0.060 & 50.000000 \\
TRPO-Lagrange & Intersection & 0.346 ± 0.069 & 0.751 ± 0.021 & 0.656 ± 0.055 & NaN \\
Vanilla-MCTS & Intersection & 0.405 ± 0.090 & 0.724 ± 0.028 & 0.593 ± 0.077 & NaN \\
W-MCTS & Intersection & \textbf{0.807 ± 0.051} & \textbf{0.270 ± 0.020} & \textbf{0.181 ± 0.049} & \textbf{36.0} \\ \midrule
C-MCTS & Racetrack & 0.536 ± 0.065 & 0.565 ± 0.025 & 0.449 ± 0.067 & 41.000000 \\
CPPO & Racetrack & 0.297 ± 0.060 & 0.788 ± 0.018 & 0.703 ± 0.051 & NaN \\
CVaR-MCTS & Racetrack & 0.760 ± 0.041 & 0.305 ± 0.017 & 0.226 ± 0.044 & 29.000000 \\
Risk-Averse MCTS & Racetrack & 0.664 ± 0.056 & 0.424 ± 0.022 & 0.318 ± 0.061 & 35.000000 \\
TRPO-Lagrange & Racetrack & 0.343 ± 0.064 & 0.742 ± 0.017 & 0.651 ± 0.054 & NaN \\
Vanilla-MCTS & Racetrack & 0.412 ± 0.079 & 0.702 ± 0.026 & 0.575 ± 0.076 & NaN \\
W-MCTS & Racetrack & \textbf{0.833 ± 0.045} & \textbf{0.240 ± 0.020} & \textbf{0.149 ± 0.047} & \textbf{24.0} \\
\bottomrule\bottomrule
\end{tabular}
\caption{Comparison of steady‑state performance and convergence speed across four traffic scenarios (Highway, Roundabout, Intersection, Racetrack) for seven algorithms. Metrics reported are average reward ($\uparrow$), 90\%‑CVaR ($\downarrow$), average cost ($\downarrow$), and steps to reach 95\% of final reward ($\downarrow$); values are mean $\pm$ std over 30 runs, with best results in bold.}
\label{tab:overall}
\end{table*}

Figure~\ref{fig:all} (in Appendix) and Table~\ref{tab:overall} show the performance trends of different algorithms over time. 
Taking the Highway environment as an example (Figure~\ref{fig:all}(a), (e), (i)), the W-MCTS algorithm quickly converges to a high and stable reward (around 0.836), while both cost and risk (CVaR) drop rapidly and stabilize at low levels (0.148 and 0.236, respectively). 
In contrast, the Vanilla-MCTS, TRPO-Lagrange, C-MCTS, Risk-Averse MCTS, and CPPO algorithms show much lower rewards, and their cost and CVaR remain high for a long duration, indicating clear instability.

The experimental results in the Intersection, Racetrack, and Roundabout environments show similar trends to the Highway environment (Table~\ref{tab:overall}). 
W-MCTS consistently demonstrates superior stability and performance, especially in reducing risk and cost. 
In addition, CVaR-MCTS also performs well, second only to W-MCTS, and clearly outperforms the other methods. 
This indicates that directly optimizing CVaR is indeed effective in reducing risk.

In summary, effectively balancing reward and risk control, CVaR-MCTS and W-MCTS significantly enhance safety and stability, with an additional advantage of W-MCTS due to its uncertainty modeling as it is evident from Table~\ref{tab:overall}.

%% file: 06Conclusion.tex
\vspace{-0.1in}
\section{Conclusion and Future Work}
The proposed CVaR-MCTS and W-MCTS are the first to achieve both strict PAC tail-risk safety guarantees and sublinear regret bounds within Monte Carlo Tree Search. By incorporating Conditional Value-at-Risk into the UCT selection criterion and correcting finite-sample bias using first-order Wasserstein ambiguity sets, the algorithms manage to control extreme loss risks while maintaining a convergence rate comparable to that of the classical UCT.
Extensive experiments demonstrate that both algorithms significantly outperform existing risk-sensitive methods in terms of average return, CVaR, convergence speed, and stability, validating the effectiveness of the theoretical analysis. Future work will further
extend the approach to continuous action spaces and explore distributional models to improve tail-risk estimates under more challenging non-stationary environments.

%% file: 07appendix.tex
\clearpage
\newpage
\appendix

\section{Proof}
\subsection{Proof of Theorem~\ref{thm:cvar_safety}}

\begin{proof}
First, we leverage the following conic representation of CVaR~\cite{rockafellar2000optimization}.
\begin{equation}
\label{eqn:cvar_2}
\begin{aligned}
\text{CVaR}_{\alpha}(Z) = \min_{v\in\mathbb{R}}\left\{v+\frac{1}{1-\alpha}\mathbb{E}\left[\left(Z-v\right)^{+}\right]\right\}
\end{aligned}
\end{equation}

where $(x)^{+} = \max\{x, 0\}$. Accordingly, the empirical estimate for the samples $\{Z_{i}\}_{i=1}^{N}$ is given by:

\begin{equation}
\label{eqn:cvar_3}
\begin{aligned}
    \widehat{\text{CVaR}}_{\alpha}(Z) = \min_{v\in\mathbb{R}}\left\{v+\frac{1}{(1-\alpha)N}\sum_{i=1}^{N}\left(Z_i-v\right)^{+}\right\}
\end{aligned}
\end{equation}
For the cost dimension $k$ , the tail cost can be obtained as:
\begin{equation}
\begin{aligned}
    Z_{i} = \frac{1}{N} \sum_{t=0}^{N-1} c^{(k)}(s_t^{(i)},a_{t}^{(i)})\in[0,1]
\end{aligned}
\end{equation}

By Hoeffding’s inequality, for any $\varepsilon > 0$, it holds that:
\begin{equation}
\begin{aligned}
    \Pr \left[\left|\frac{1}{N}\sum_{i=1}^{N}\left(Z_i-v\right)^{+} -\mathbb{E}\left[\left(Z-v\right)^{+}\right] \right| \geq \varepsilon\right] \leq 2 \exp(-2 N \varepsilon^2)
\end{aligned}
\end{equation}

To ensure that the overflow probability across all $K$ components is controlled within $\delta$, set $\varepsilon = \beta_{C}\sqrt{\ln(2K/\delta)\,/\, (2N)}$ and let $N = H$. Then:
\begin{equation}
\begin{aligned}
&2\exp(-2N\varepsilon^2)\\
=&2\exp(-2N \cdot \frac{\beta_{C}^{2}\ln(\frac{2K}{\delta})}{2N})\\
=&2\exp(-\beta_{C}^{2}\ln \frac{2K}{\delta})\\
=&2(\frac{2K}{\delta})^{-\beta_{C}^2}
\end{aligned}
\end{equation}

Since $\beta_{C}^{2} = 2$, we obtain:
\begin{equation}
\begin{aligned}
2(\frac{2K}{\delta})^{-\beta_{C}^2}=\frac{2\delta^2}{(2K)^2}=\frac{\delta^2}{2K^2}
\end{aligned}
\end{equation}

So
\begin{equation}
\begin{aligned}
 \Pr \left[\left|\frac{1}{N}\sum_{i=1}^{N}\left(Z_i-v\right)^{+} -\mathbb{E}\left[\left(Z-v\right)^{+}\right] \right| 
\geq  \beta_{C}\sqrt{\frac{\ln(2K/\delta)}{2N}}\right]
\leq \frac{\delta^2}{2K^2}
\end{aligned}
\end{equation}

From the minimization forms of Eqn~\ref{eqn:cvar_2} and Eqn~\ref{eqn:cvar_3},
and noting that for any fixed $v \in \mathbb{R}$, the following inequality holds:
\begin{equation}
\begin{aligned}
\min_vf(v) \leq f(v), \min_{v} \hat{f}(v) \leq \hat{f}(v)
\end{aligned}
\end{equation}
Taking the difference and then the absolute value, and exchanging the order for the $\pm$ directions respectively, we obtain:
\begin{equation}
\begin{aligned}
\left|\min_vf(v) - \min_{v} \hat{f}(v)\right| \leq \sup_{v}\left|f(v) - \hat{f}(v)\right| 
\end{aligned}
\end{equation}
Finally, substituting this into the expression, we obtain:
\begin{equation}
\label{eqn:cvar_diff}
\begin{aligned}
\bigl|\widehat{\text{CVaR}}_{\alpha}(Z)-\text{CVaR}_{\alpha}(Z)\bigr|
\;\le\;
\frac{1}{1-\alpha}\,
\max_{v\in\mathbb{R}}
\left| \frac1N\sum_{i=1}^N\!\bigl(Z_i-v\bigr)^{+}
-\mathbb{E}\!\left[\bigl(Z-v\bigr)^{+}\right]
\right|  
\end{aligned}
\end{equation}

Note that $(1 - \alpha)^{-1} \ge 1$, so this amplification factor must be retained when applying the Hoeffding inequality. The resulting probability bound is updated as:
\begin{equation}
\begin{aligned}
\Pr\!\left[
\left|\widehat{\text{CVaR}}_{\alpha}^{(k)}
-\text{CVaR}_{\alpha}^{(k)}\right|
\;\ge\;
\frac{\beta_C}{1-\alpha}\sqrt{\frac{\ln(2K/\delta)}{2H}}
\right]
\;\le\;
\frac{\delta^{2}}{2K^{2}}.
\end{aligned}
\end{equation}

We further obtain:

\begin{equation}
\begin{aligned}
\Pr\!\left[
\widehat{\text{CVaR}}_{\alpha}^{(k)}
+
\frac{\beta_C}{1-\alpha}\sqrt{\frac{\ln(2K/\delta)}{2H}}
<
\text{CVaR}_{\alpha}^{(k)}
\right]
\;\le\;
\frac{\delta^{2}}{2K^{2}}.
\end{aligned}
\end{equation}
Finally, we apply the union bound.
\begin{equation}
\begin{aligned}
\Pr\!\Bigl[
\forall k:\;
\text{CVaR}_{\alpha}^{(k)}
\le
\widehat{\text{CVaR}}_{\alpha}^{(k)}
+
\frac{\beta_C}{1-\alpha}\sqrt{\frac{\ln(2K/\delta)}{2H}}
\Bigr]
\;\ge\;
1-\delta.
\end{aligned}
\end{equation}

If the algorithm guarantees at each node of the search tree that:
\begin{equation}
\begin{aligned}
\widehat{\text{CVaR}}_{\alpha}^{(k)}\;+\;
\frac{\beta_C}{1-\alpha}\sqrt{\frac{\ln(2K/\delta)}{2H}}
\;\le\;\tau_k,
\end{aligned}
\end{equation}

Then the true CVaR constraint $\text{CVaR}_{\alpha}^{(k)} \le \tau_k$ holds with confidence level $1 - \delta$.
To achieve the same confidence level, the required sample size (or simulation depth) must be correspondingly scaled up to:

\begin{equation}
\begin{aligned}
H \;\ge\;
\frac{2\beta_C^{2}}{(1-\alpha)^{2}}
\cdot
\frac{\ln(2K/\delta)}{\varepsilon^{2}},
\end{aligned}
\end{equation}
where $\varepsilon$ denotes the allowable upper bound on the estimation error.

\end{proof}

\subsection{Proof of Theorem~\ref{thm:param_convergence}}
\begin{proof}
Here, we write the update as:
\begin{equation}
\begin{aligned}
\boldsymbol{\lambda}_{t+1} = \left[\boldsymbol{\lambda}_t + \eta_t(g(\boldsymbol{\lambda}_t)+\xi_t)\right]_{+}
\end{aligned}
\end{equation}
Where $\xi_t = \hat{g}(\boldsymbol{\lambda}_t) - g(\boldsymbol{\lambda}_t)$ denotes the zero-mean noise.

To prove convergence, we need to verify the Robbins–Monro conditions.
First, due to the unbiasedness of trajectory sampling, we have $\mathbb{E}\left[\hat{g}(\lambda_t)\,|\,\lambda_t\right] = g(\lambda_t)$. Therefore, $\mathbb{E}\left[\xi_t\,|\,\mathcal{F}_t\right] = 0$, where $\xi_t$ denotes the sampling noise.

Moreover, since the one-step cost lies within $[0,1]$, the cumulative CVaR is also bounded. Therefore, there exists a constant $M < \infty$ such that
$
\mathbb{E}\left[\|\xi_t\|^2 \mid \mathcal{F}_t\right] \leq M.
$

Since the dual objective $\mathcal{L}(\lambda)$ is convex and its gradient is given by the dual representation of CVaR (a convex combination), it can be shown that there exists $L > 0$ such that
$
\|g(\lambda) - g(\lambda^{\prime})\| \leq L \|\lambda - \lambda^{\prime}\|.
$
Therefore, Lipschitz continuity can be verified.

According to stochastic approximation theory, under the above conditions, the projected stochastic approximation iteration almost surely converges to the dual optimal solution $\lambda^{*}$. Therefore, we need to provide an error analysis.

Let $\Delta_t = \boldsymbol{\lambda}_t - \boldsymbol{\lambda}^{*}$. By the non-expansiveness property of projection, we have:
\begin{equation}
\begin{aligned}
\|\Delta_{t+1} \|^{2} &\leq \|\Delta_t + \eta_t(g(\boldsymbol{\lambda_t})+\xi_t) \|^2\\
&=\|\Delta_{t} \|^{2} + 2\eta_t\Delta_t^{\top}g(\boldsymbol{\lambda_t}) + 2\eta_t\Delta_t^{\top}\xi_t + \eta_t^2\|g(\boldsymbol{\lambda_t})+\xi_t\|^2\\
&\leq \|\Delta_{t} \|^{2} + 2\eta_t\Delta_t^{\top}g(\boldsymbol{\lambda_t}) + 2\eta_t\Delta_t^{\top}\xi_t  + \eta_t^2(2\|g(\boldsymbol{\lambda_t})\|^2 + 2\|\xi_t^2\|)\\
&\leq \|\Delta_{t} \|^{2} + 2\eta_t\Delta_t^{\top}g(\boldsymbol{\lambda_t}) + 2\eta_t\Delta_t^{\top}\xi_t + \eta_t^2(G^2+2\|\xi_t\|^2)
\end{aligned}
\end{equation}
where $G$ is a bounded constant.

Since $D(\lambda)$ is convex and $\lambda^{*}$ is the optimal solution, it follows from the inner product bound under convexity that:
\begin{equation}
\begin{aligned}
&\Delta_t^{\top}g(\boldsymbol{\lambda_t})\\
=& (\boldsymbol{\lambda}_t - \boldsymbol{\lambda}^{*})^{\top} \nabla D(\boldsymbol{\lambda}_t)\\
\geq & D(\boldsymbol{\lambda_t}) -D(\boldsymbol{\lambda}^{*})\\
\geq & 0
\end{aligned}
\end{equation}

Since the noise term is unbiased, we have
$
\mathbb{E}\left[\Delta_t^{\top}\,\xi_t \,\big|\, \mathcal{F}_t\right] = 0.
$

Taking the conditional expectation of the above expansion and applying the unbiasedness and step-size conditions:
\begin{equation}
\begin{aligned}
\mathbb{E}[\|\Delta_{t+1}\|^2|\mathcal{F}_t] \leq \|\Delta_t\|^2 + \eta_t^2(G^2+2M)
\end{aligned}
\end{equation}

More precisely, by recursion:
\begin{equation}
\begin{aligned}
    \mathbb{E}[\|\Delta_{t+1}\|^2] \leq (1)\mathbb{E}[\|\Delta_t\|^2] + C\eta^2
\end{aligned}
\end{equation}

Let $\eta_t = t^{-\gamma}$ with $\gamma \in (1/2, 1]$. It can be derived that
$
\mathbb{E}\bigl[\|\Delta_t\|^2\bigr] = \mathcal{O}\bigl(t^{-\min(2\gamma -1,\,1)}\bigr).
$Taking the commonly used $\gamma = 1$, we have:

\begin{equation}
\begin{aligned}
    \mathbb{E}[\|\Delta_t\|^2] = \mathcal{O}(t^{-1})\Rightarrow \|\Delta_t\|_1 \leq \sqrt{K}\|\Delta_t\|_2 = \mathcal{O}(t^{-1/2})
\end{aligned}
\end{equation}

\end{proof}

\subsection{Proof of Theorem~\ref{thm:regret_cvar}}

\begin{proof}
Let $\mu(s,a) = \mathbb{E}[R_H \mid (s,a)]$ denote the true expected return starting from node $(s,a)$ and following the true optimal policy $\pi^\star$, where $R_H = \sum_{t=0}^{H-1} r(s_t, a_t)$.
Let the optimal action be denoted as $a^\star(s) = \arg\max_a \mu(s, a)$, and define the gap $\Delta(s, a) = \mu(s, a^\star(s)) - \mu(s, a)$, where $\Delta > 0$ if and only if $a \neq a^\star(s)$.
Define the cumulative time steps $T$ as the total number of rollouts (which also corresponds to the number of visits to the root node).
Let $N_t(s)$ and $N_t(s, a)$ denote the number of times node $s$ and edge $(s,a)$ have been visited before time $t$, respectively.

First, we present the result based on the Hoeffding inequality.
\begin{equation}
\begin{aligned}
\Pr\left(|\hat{Q}_t(s,a) - \mu(s,a)|>\beta_{R}\sqrt{\frac{\ln N_t(s)}{1+N_t(s,a)}}\right) \leq N_t(s)^{-2}
\end{aligned}
\end{equation}
By setting $\beta_R = \sqrt{2}$ and applying a union bound over all $(s,a)$ pairs and all time steps, we obtain the event $\mathcal{E}_R$ 
\begin{equation}
\begin{aligned}
\mathcal{E}_R =\left\{\forall t, \forall(s,a):\left|\hat{Q}_{t}(s,a) - \mu(s,a)\right|\leq \beta_R\sqrt{\frac{\ln N_t(s)}{1+N_t(s,a)}}\right\}  
\end{aligned}
\end{equation}
which holds with probability at least $\Pr(\mathcal{E}_R) \ge 1 - \tfrac{\pi^2}{3T^2}$.

Similarly, for single-step costs $c \in [0,1]$, Hoeffding’s inequality yields the following bound for the CVaR estimate:
\begin{equation}
\begin{aligned}
\Pr\left(|\widehat{\text{CVaR}}_t(s,a) - \text{CVaR}_{\alpha}(s,a)|>\beta_{R}\sqrt{\frac{\ln N_t(s)}{1+N_t(s,a)}}\right) \\ \leq N_t(s)^{-2}
\end{aligned}
\end{equation}

By setting $\beta_C = \sqrt{2}$, we define the event $\mathcal{E}_C$;
taking the union of the two events yields:

\begin{equation}
\begin{aligned}
\Pr(\mathcal{E}_R\cap \mathcal{E}_C) \geq 1 - \frac{2\pi^2}{3T^2}
\end{aligned}
\end{equation}
Subsequently, we perform the analysis within this high-probability event, and finally account for the failure probability separately.

We first focus on a single node, and then sum up the regret over all nodes.

Within the event $\mathcal{E}_R \cap \mathcal{E}_C$, whenever the algorithm selects a suboptimal action $a \neq a^\star(s)$ at node $s$, it must satisfy:

\begin{equation}
\begin{aligned}
    U_t(s,a) \geq U_t(s,a^{*}(s))
\end{aligned}
\end{equation}

In our algorithm, both the Lagrange multiplier $\boldsymbol{\lambda}_s$ and the remaining budget $\textbf{B}_s$ at node $s$ are stored and updated in a "node-level" manner: after each rollout, we perform projected stochastic gradient updates only on the $\boldsymbol{\lambda}_s$ corresponding to the visited node slot. This update is based on the tail-risk of the entire path starting from $s$ and does not involve the first-step action. Meanwhile, $\textbf{B}_s$ is solely determined by the historical cost from the root to $s$, and is likewise independent of subsequent action choices. Therefore, when comparing $U_t(s,a)$ and $U_t(s,a^{*}(s))$, the terms $\boldsymbol{\lambda}_s$ and $\textbf{B}_s$ take exactly the same values on both sides of the inequality and cancel out. The remaining difference comes solely from the individual estimates of $Q$ and $\hat{\text{CVaR}}$ for each action.
Expanding, we obtain:

\begin{equation}
\begin{aligned}
&\hat{Q}_t(s,a) + \beta_R\sqrt{\frac{\ln N_t(s)}{1+N_t(s,a)}} - \boldsymbol{\lambda}_s^{\top}\left(\hat{\text{CVaR}}_t(s,a) + \beta_C\sqrt{\frac{\ln N_t(s)}{1+N_t(s,a)}})\textbf{1} - \textbf{B}_s\right)\\
\geq& \hat{Q}_t(s,a^*) + \beta_R\sqrt{\frac{\ln N_t(s)}{1+N_t(s,a)}} - \boldsymbol{\lambda}_s^{\top}\left(\hat{\text{CVaR}}_t(s,a^{*}) + \beta_C\sqrt{\frac{\ln N_t(s)}{1+N_t(s,a^{*})}})\textbf{1} - \textbf{B}_s\right)
\end{aligned}
\end{equation}

After canceling out the common term $-\boldsymbol{\lambda}_s^\top(-\mathbf{B}_s)$, we obtain:
\begin{equation}
\begin{aligned}
&\hat{Q}_t(s,a) - \hat{Q}_t(s,a^{*}) \geq -\beta_R\Delta_R + \boldsymbol{\lambda}_s^{\top} \left(\hat{\text{CVaR}}_t(s,a) - \hat{\text{CVaR}}_t(s,a^{*}) + \beta_C\Delta_C\right)
\end{aligned}
\end{equation}
where $\Delta_R = \sqrt{\frac{\ln N_t(s)}{1+ N_t(s,a)}} - \sqrt{\frac{\ln N_t(s)}{1+N_t(s,a^*)}}$, $\Delta_C = \sqrt{\frac{\ln N_t(s)}{1+ N_t(s,a)}} - \sqrt{\frac{\ln N_t(s)}{1+N_t(s,a^*)}}$

According to the Hoeffding event $\mathcal{E}_R \cap \mathcal{E}_C$,
\begin{equation}
\begin{aligned}
&\left|\hat{Q}_t(s,a) - \mu(s,a)\right| \leq \beta_R\sqrt{\frac{\ln N_t(s)}{1+N_t(s,a)}}\\
&\left|\hat{\text{CVaR}}_t(s,a) - \text{CVaR}_{\alpha}(s,a)\right|\leq \beta_C \sqrt{\frac{\ln N_t(s)}{1+N_t(s,a)}}
\end{aligned}
\end{equation}

Substituting in and simplifying, we obtain an upper bound on the gap in true expected returns:

\begin{equation}
\begin{aligned}
&\mu(s,a^{*}) - \mu(s,a)\\ 
=& [\mu(s,a^{*})-\hat{Q}_t(s,a^{*})] + [\hat{Q}_t(s,a^{*}) - \hat{Q}_t(s,a)] + [\hat{Q}_t(s,a) -\mu(s,a)]\\
\leq& 2\beta_R\sqrt{\frac{\ln N_t(s)}{1+ N_t(s,a)}} + \beta_R \Delta_R - \boldsymbol{\lambda}_s^{\top}(\hat{\text{CVaR}}_{\alpha,t}(s,a) - \hat{\text{CVaR}}_{\alpha,t}(s,a)+\beta_C\Delta_C) \\
\leq & 3\beta_R\sqrt{\frac{\ln N_t(s)}{1+ N_t(s,a)}} -\boldsymbol{\lambda}_s^{\top}(\hat{\text{CVaR}}_{\alpha,t}(s,a) - \hat{\text{CVaR}}_{\alpha,t}(s,a)+\beta_C\Delta_C)\\
\leq & 3\beta_R\sqrt{\frac{\ln N_t(s)}{1+ N_t(s,a)}} +||\boldsymbol{\lambda}_s||_1||\hat{\text{CVaR}}_{\alpha,t}(s,a) - \hat{\text{CVaR}}_{\alpha,t}(s,a)+\beta_C\Delta_C||_{\infty}  \\
\leq& 3\beta_R\sqrt{\frac{\ln N_t(s)}{1+ N_t(s,a)}} + ||\boldsymbol{\lambda}_s||_1 \cdot 2\beta_C\sqrt{\frac{\ln N_t(s)}{1+N_t(s,a)}}\\
=& 2\left(\frac{3}{2}\beta_R + \||\boldsymbol{\lambda}_s||_1 \beta_C \right)\sqrt{\frac{\ln N_t(s)}{1+N_t(s,a)}}\\
=& 2 \hat{\beta}\sqrt{\frac{\ln N_t(s)}{1+N_t(s,a)}}
\end{aligned}
\end{equation}

which is:
\begin{equation}
\begin{aligned}
N_t(s,a)\leq \frac{4\hat{\beta}^2 \ln N_t(s)}{\Delta(s,a)^2}
\end{aligned}
\end{equation}

For any time step $t$, we have $N_t(s) \le T$; therefore,

\begin{equation}
\begin{aligned}
N_T(s,a)\leq \frac{4\hat{\beta}^2 }{\Delta(s,a)^2}\ln T
\end{aligned}
\end{equation}
This yields an upper bound on the number of visits to the suboptimal edge $(s, a)$.
The entire search tree contains at most $|\mathcal{S}| \times |\mathcal{A}|$ edges, but due to the depth limit $H$, any single rollout can visit at most $H$ edges.
By summing the number of visits to all suboptimal edges across all nodes, we obtain:
\begin{equation}
\begin{aligned}
\sum_{s}\sum_{a\neq a^{*}(s)}N_T(s,a)\leq 4\hat{\beta}^2 \ln T \sum_{s}\sum_{a\neq a^{*}(s)}\frac{1}{\Delta(s,a)^2}
\end{aligned}
\end{equation}
In classical analysis, the right-hand side is treated as a constant $\kappa$ (which depends only on the problem instance, not on $T$).
As a result, the total number of suboptimal selections is bounded by $\kappa \ln T$.

Let the per-rollout regret be at most $H$; then,
\begin{equation}
\begin{aligned}
\mathcal{R}_T = \mathbb{E}\left[\sum_{t=1}^{T}(\mu^{*}- R_t)\right]\leq H\cdot (\kappa \ln T)=\mathcal{O}(H\ln T)
\end{aligned}
\end{equation}

However, we also need to account for the additional optimistic bias introduced by the risk-related confidence term:
each time a suboptimal edge is visited, the reward side receives an inflated bonus of at most
$
\hat{\beta} \sqrt{ \tfrac{\ln N_t(s)}{1 + N_t(s,a)} } \le \hat{\beta} \sqrt{ \ln T }.
$
After at most $\kappa \ln T$ such visits, the cumulative contribution is bounded by
$
\hat{\beta} \sqrt{\ln T} \cdot \kappa \ln T = \kappa \hat{\beta} (\ln T)^{3/2}.
$

Combining this term with the previous regret from suboptimal selections, and adding at most an additional loss of $H$ due to the failure event with probability $2\pi^2 / (3T)$, we obtain:
\begin{equation}
\begin{aligned}
\mathcal{R}_T \leq H \kappa \ln T + \kappa \hat{\beta} (\ln T)^{3/2} +H \cdot \frac{2\pi^2}{3}
\end{aligned}
\end{equation}

When $T$ is sufficiently large, we have $(\ln T)^{3/2} \le \sqrt{T \ln T}$.
By consolidating the constants, we arrive at the stated result.

\begin{equation}
\begin{aligned}
\mathcal{R}_T \leq c \sqrt{T\ln T}, c =\mathcal{O}(\beta_R+ \beta_CK + H)
\end{aligned}
\end{equation}

Where the term $\beta_C K$ arises because the risk confidence bound is amplified by the Lagrange multiplier $|\boldsymbol{\lambda}_s|_1$, which can be at most $K$.
This affects only the constant coefficient, not the asymptotic order.

\end{proof}

\subsection{Proof of Theorem~\ref{thm:w_safety}}
\begin{proof}
We need to prove three key lemmas.
\begin{Lemma}[Wasserstein-Lipschitz Property of CVaR]
\label{lem:cvar_w_lipschitz}
Let the random variable $Z \in [0,1]$ have distributions $P$ and $Q$, respectively. Given a confidence level $\alpha \in (0,1)$, we have:
\begin{equation}
\begin{aligned}
    \left|\text{CVaR}_{\alpha}^{P}(Z) - \text{CVaR}_{\alpha}^{Q}(Z)\right| \leq \frac{1}{1-\alpha}W_1(P,Q) =: L_C W_1(P,Q)
\end{aligned}
\end{equation}
\end{Lemma}
\begin{proof}
(Proof of Lemma~\ref{lem:cvar_w_lipschitz})
We can write the dual representation of CVaR as:
\begin{equation}
\begin{aligned}
    \text{CVaR}_{\alpha}^{P}(Z) = \min_{\mu\in \mathbb{R}}\left\{\mu + \frac{1}{1-\alpha}\mathbb{E}_{P}\left[(Z-\mu)_{+}\right]\right\}
\end{aligned}
\end{equation}
Therefore, we can construct a family of Lipschitz functions by setting:
\begin{equation}
\begin{aligned}
    f_{\mu}(z) := \mu + \frac{1}{1-\alpha}(z-\mu)_{+}
\end{aligned}
\end{equation}
For $\forall z_1,z_2$, we have $|f_{\mu}(z_1) - f_{\mu}(z_2)| = \frac{1}{1-\alpha}|(z_1 - \mu)_{+} - (z_2-\mu)_{+}|\leq \frac{1}{1-\alpha}|z_1-z_2|$. So Lipschitz const number is $L_C = \frac{1}{1-\alpha}$

According to the Kantorovich–Rubinstein duality formula, we obtain:
\begin{equation}
\label{eqn:w-dual}
\begin{aligned}
    W_1 (P,Q) = \sup_{g:\|g\|_{\text{Lip}}\leq 1} \left|\mathbb{E}_{P}g(Z) -\mathbb{E}_{Q} g(Z)\right|
\end{aligned}
\end{equation}

Therefore, for each $\mu$, we have:

\begin{equation}
\begin{aligned}
    &\left|\mathbb{E}_{P}f_\mu - \mathbb{E}_{Q}f_\mu\right| = \frac{1}{1-\alpha}\left|\mathbb{E}_{P}g_\mu - \mathbb{E}_{Q}g_\mu\right|\\
\leq &\frac{1}{1-\alpha}W_1 (P,Q) = L_C W_1(P,Q)
\end{aligned}
\end{equation}

We set $F_{P}(\mu):= \mathbb{E}_{P}f_\mu(Z),F_{Q}(\mu):= \mathbb{E}_{Q}f_\mu(Z)$.

So we can get:
\begin{equation}
\begin{aligned}
    \text{CVaR}_{\alpha}^{P}(Z) = \inf_{\mu} F_{P}(\mu),\text{CVaR}_{\alpha}^{Q}(Z) = \inf_{\mu} F_{Q}(\mu)
\end{aligned}
\end{equation}

So
\begin{equation}
\begin{aligned}
&\left|\text{CVaR}_{\alpha}^{P}(Z) - \text{CVaR}_{\alpha}^{Q}(Z)\right|\\
=&\left|\inf_{\mu} F_{P}(\mu) - \inf_{\mu} F_{Q}(\mu)\right|\\
\leq & \sup_{\mu}|F_P(\mu)-F_Q(\mu)| \\
\leq & L_C W_1(P,Q) \leq \frac{1}{1-\alpha} W_1 (P,Q)
\end{aligned}
\end{equation}

\end{proof}

\begin{Lemma}
(Wasserstein Concentration from Empirical Distribution to True Distribution)
\label{lem:w_coverage}
For i.i.d. random variables $Z_1, \ldots, Z_T \in [0,1]$, let the true distribution be $P$ and the empirical distribution be $\hat{P}_{T} := \frac{1}{T}\sum_{i=1}^{T}\delta_{Z_i}$. Then, for any $\varepsilon > 0$:
\begin{equation}
\begin{aligned}
    \Pr\left[W_1(P,\hat{P}_{T})>\varepsilon\right] \leq 2\exp(-2T\varepsilon^2)
\end{aligned}
\end{equation}
\end{Lemma}
\begin{proof}
(Proof of Lemma~\ref{lem:w_coverage})
According to~\ref{eqn:w-dual}, we will express the Wasserstein distance as a Lipschitz empirical process.
\begin{equation}
\begin{aligned}
    \Delta_T := W_1(P,\hat{P}_T) =\sup_{\|g\|_{\text{Lip}}\leq 1} |\hat{\mathbb{E}}_{T}g - \mathbb{E}_P g|
\end{aligned}
\end{equation}
where $\hat{\mathbb{E}}_{T}$ is sample mean.

Next, we discretize the unit Lipschitz ball by setting $\eta = \frac{\varepsilon}{2}$. Then, on $[0,1]$, the unit Lipschitz ball can be covered by an $\eta$-net $\mathcal{N}_{\eta} = \{g_j\}_{j=1}^{M}$, with $M \le \lceil \frac{2}{\eta} \rceil$. For any function $g$, there exists an approximation $g_j$ within distance $\le \eta$.

Therefore, we first focus on the single-function Hoeffding bound. Fixing some $g$, we have:
\begin{equation}
\begin{aligned}
    \Pr[|\hat{\mathbb{E}}_{T}g - \mathbb{E}_P g|>\varepsilon-\eta] \leq 2\exp(-2T(\varepsilon-\eta)^2)
\end{aligned}
\end{equation}
Next, by taking the union over all grid points and enlarging $\eta$, we apply the union bound over $\mathcal{N}_{\eta}$.
\begin{equation}
\begin{aligned}
    \Pr[\exists g_j:|\hat{\mathbb{E}}_{T}g - \mathbb{E}_P g|>\varepsilon-\eta] \leq 2M\exp(-2T(\varepsilon-\eta)^2)
\end{aligned}
\end{equation}
We set $\eta = \varepsilon/2$, can get $M\leq 4/\varepsilon$ and $(\varepsilon-\eta)=\varepsilon/2$. So
\begin{equation}
\begin{aligned}
    \Pr\left[\Delta_T>\varepsilon\right]\leq \frac{8}{\varepsilon}\exp(-\frac{1}{2}T\varepsilon^2)
\end{aligned}
\end{equation}
since $\varepsilon\in (0,1]$, $\frac{8}{\varepsilon}\leq 16\leq \exp(T\varepsilon^2)$, when $T\geq 2$, we can get:
\begin{equation}
\begin{aligned}
    \Pr[\Delta_T>\varepsilon] \leq 2\exp(-2T\varepsilon^2)
\end{aligned}
\end{equation}
\end{proof}

\begin{Lemma}
(Sampling Confidence Bound for Empirical CVaR)
\label{lem:cvar_sample}
For $Z_1, \ldots, Z_T \overset{\text{i.i.d.}}{\sim} P \subset [0,1]$, let the order statistics in ascending order be $Z_{(1)} \le \cdots \le Z_{(T)}$. Define $k = \lceil \alpha T \rceil$. Then, the sample CVaR is given by
$
\widehat{\text{CVaR}}_{\alpha} := \frac{1}{m}\sum_{i=k+1}^{T} Z_{(i)}.
$
Setting $\psi_{T} := \beta_C \sqrt{\frac{\ln T}{1 + T}}$ with $\beta_C = \sqrt{2}$, we have:
\begin{equation}
\begin{aligned}
    \Pr\left[|\widehat{\text{CVaR}}_\alpha - \text{CVaR}_\alpha^{P}|\leq \psi_T\right]\geq 1 - \frac{2}{T^2}
\end{aligned}
\end{equation}
\end{Lemma}
\begin{proof}
(Proof of Lemma~\ref{lem:cvar_sample})
we can write:
\begin{equation}
\begin{aligned}
    \text{CVaR}_{\alpha}^{P} = \frac{1}{1-\alpha}\int_{\alpha}^{1} F^{-1}(u)du
\end{aligned}
\end{equation}
where $F^{-1} = \inf\{z:F(z)\geq u\}$

Therefore, the sample VaR $\hat{v} := Z_{(k)}$ and the true VaR $v_{\alpha} := F^{-1}(\alpha)$ satisfy the Kolmogorov–Smirnov (DKW) inequality:
\begin{equation}
\begin{aligned}
    \Pr\left[|\hat{v}-v_{\alpha}|>\varepsilon\right]\leq 2\exp(-2T\varepsilon^2)
\end{aligned}
\end{equation}

Consider the tail indicator $\mathbf{1}_{\{Z_i \ge \hat{v}\}} \in \{0,1\}$. Under the fixed threshold $\hat{v}$, the number of tail samples $m$ follows a binomial distribution. Therefore, for any $\xi > 0$, by Hoeffding’s inequality:
\begin{equation}
\begin{aligned}
\Pr\left[|\widehat{\text{CVaR}}_{\alpha} - \mathbb{E}_{P}[Z|Z\geq \hat{v}]|\geq \xi|\hat{v}\right] \leq 2\exp(-2m\xi^2)
\end{aligned}
\end{equation}
we set $\xi = \beta_C\sqrt{\frac{\ln T}{T}},\beta_C =\sqrt{2}$ and $m\geq (1-\alpha)T \geq (1-\alpha)$. we can get:
\begin{equation}
\begin{aligned}
    2\exp(-2m\xi^2)\leq 2\exp(-2T\xi^2) = 2T^{-2}
\end{aligned}
\end{equation}
Incorporating the additional error caused by the VaR estimation error $|\hat{v}-v_{\alpha}|$, we can obtain the overall probabilistic bound.
\begin{equation}
\begin{aligned}
    \Pr\left[|\widehat{\text{CVaR}}_\alpha - \text{CVaR}_\alpha^{P}|\leq \psi_T\right]\geq 1 - \frac{2}{T^2}
\end{aligned}
\end{equation}
\end{proof}

For any $\widetilde{P} \in \mathcal{B}_W(\hat{P}, \varepsilon_0)$, it follows from the triangle inequality and Lemma~\ref{lem:cvar_w_lipschitz} that:
\begin{equation}
\begin{aligned}
    \text{CVaR}_{\alpha}^{\widetilde{P}}(C_H) \leq \underbrace{\text{CVaR}_{\alpha}^{\hat{P}}(C_H)}_{{(A)}} + \underbrace{L_C W_1(\widetilde{P},\hat{P})}_{(B)}
\end{aligned}
\end{equation}

We replace $(A)$ with the empirical estimate from Lemma~\ref{lem:cvar_sample}, and define the event $E_1$ as follows:
\begin{equation}
\begin{aligned}
    E_1 :=\left\{ | \text{CVaR}_{\alpha}^{\hat{P}} - \widehat{\text{CVaR}}_{\alpha}|\leq \psi_T\right\}, \Pr[E_1]\geq 1-\frac{2}{T^2}
\end{aligned}
\end{equation}

Under $E_1$:
\begin{equation}
\begin{aligned}
    \text{CVaR}_{\alpha}^{\hat{P}}(C_H) \leq \widehat{\text{CVaR}}_{\alpha}(C_H) + \psi_T
\end{aligned}
\end{equation}

We use Lemma~\ref{lem:w_coverage} to handle the Wasserstein error term (B).
\begin{equation}
\begin{aligned}
    E_2:=\left\{W_1(\widetilde{P},\hat{P})\leq \varepsilon_{s} = \frac{\varepsilon_{0}}{\sqrt{T}}\right\}, \Pr[E_2]\geq 1-2\exp(-2\varepsilon_{0})
\end{aligned}
\end{equation}

Under $E_2$
\begin{equation}
\begin{aligned}
    L_CW_1(\widetilde{P},\hat{P}) \leq L_C \varepsilon_s
\end{aligned}
\end{equation}
Therefore, combining the confidence events, we have
$
\Pr[E_1 \cap E_2] = 1 - \frac{2}{T^2} - 2\exp\bigl(-2\varepsilon_0^2\bigr).
$
We have:

\begin{equation}
\begin{aligned}
\text{CVaR}_{\alpha}^{\widetilde{P}}(C_H) \leq  \widehat{\text{CVaR}}_\alpha(C_H)+\psi_T + L_C\varepsilon_s \leq \tau
\end{aligned}
\end{equation}
\end{proof}

So
\begin{equation}
\begin{aligned}
    \Pr_{\widetilde{P}}[\text{CVaR}_{\alpha}(C_{H}) \leq \tau]\geq \Pr[E_1\cap E_2] \geq 1- \frac{2}{T^2}-2\exp(-2\varepsilon_0^2)
\end{aligned}
\end{equation}

\subsection{Proof of Theorem~\ref{thm:regret}}

\begin{proof}
W-MCTS differs from CVaR-MCTS only by an additional distributional robustness compensation term.

Within the event $\mathcal{E}_R \cap \mathcal{E}_C$, we still have:
\begin{equation}
\begin{aligned}
    \left|\hat{Q}_t(s,a) - \mu(s,a)\right|\leq \beta_{R}\sqrt{\frac{\ln N_t(s)}{1+N_t(s,a)}}
\end{aligned}
\end{equation}

However, the upper bound comparison now becomes:
\begin{equation}
\begin{aligned}
\Delta(s,a) \leq 2 \beta_R \sqrt{\frac{\ln N_t(s)}{1+N_t(s,a)}} + 2 L_C \varepsilon_0/\sqrt{N_t(s,a)}
\end{aligned}
\end{equation}

Let $f(x) = 2\beta_R\sqrt{\tfrac{\ln N}{x}} + \tfrac{2L_C\varepsilon_0}{\sqrt{x}}$, and set $x = N_t(s,a)$.
When
$
x > \frac{16\beta_R^2 \ln N}{\Delta^2} \quad \text{and} \quad x > \frac{16L_C^2 \varepsilon_0^2}{\Delta^2},
$
the inequality can no longer be triggered. Therefore,

\begin{equation}
\begin{aligned}
N_T(s,a)\leq \frac{16 \beta_R^2\ln T}{\Delta^2} + \frac{16L_C^2\varepsilon_0^2}{\Delta^2}
\end{aligned}
\end{equation}
The second term, compared to CVaR-MCTS, introduces an additional constant bias proportional to $\varepsilon_0^2$.

The regret induced by the first term is of the same order as in the previous section, namely $\mathcal{O}(\sqrt{T \ln T})$.
The second term accumulates over the entire tree up to at most $\kappa' L_C^2 \varepsilon_0^2$ times, with each incurring a loss of at most $H$, resulting in a total contribution of
$
H \kappa' L_C^2 \varepsilon_0^2.
$
Since $\varepsilon_0$ is typically chosen as a constant independent of $T$, e.g., $\mathcal{O}(T^{-1/2})$, this contribution can be rewritten as
$
c_2 L_C \varepsilon_0 \sqrt{T},
$
where $c_2 = \mathcal{O}(H)$, using the Cauchy–Schwarz inequality to separate $\sqrt{T}$ and $\varepsilon_0$.
In summary,
\begin{equation}
\begin{aligned}
\mathcal{R}_T \leq c_1\sqrt{T\ln T} + c_2 L_C\varepsilon_0 \sqrt{T}
\end{aligned}
\end{equation}

When $\varepsilon_0 = \tilde{\mathcal{O}}(T^{-1/2})$ (e.g., setting $\varepsilon_0 = \sqrt{\tfrac{\ln T}{T}}$), the second term becomes $\tilde{\mathcal{O}}(\sqrt{\ln T})$, which is of the same or lower order compared to the first term.
This implies that W-MCTS retains sublinear regret while ensuring distributional robustness.

\end{proof}

\subsection{Numerical Evaluation}

\begin{figure*}[h]
\centering
{\fontsize{10}{12}
  \subfigure[Highway reward]{\includegraphics[width=0.24\textwidth]{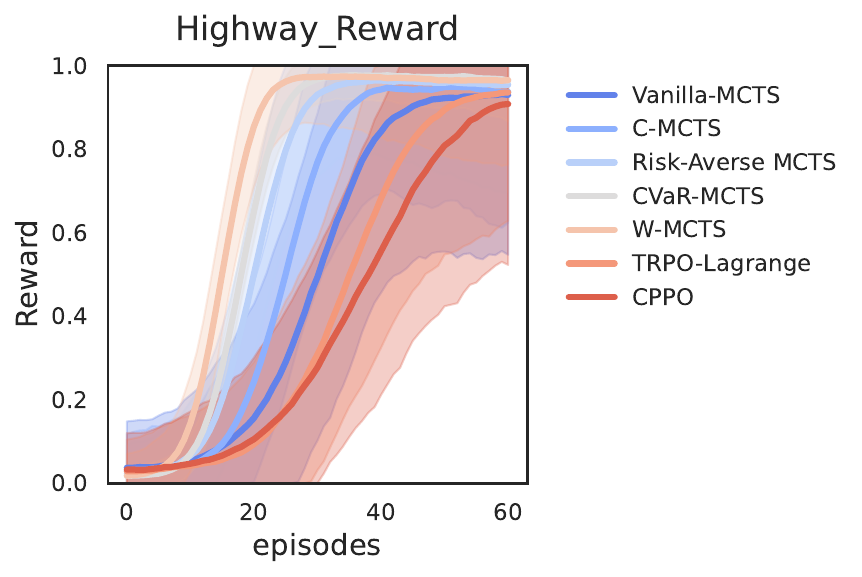}
  \label{fig:all_highway_r}
  }\hfill
  \subfigure[Intersection reward]{\includegraphics[width=0.24\textwidth]{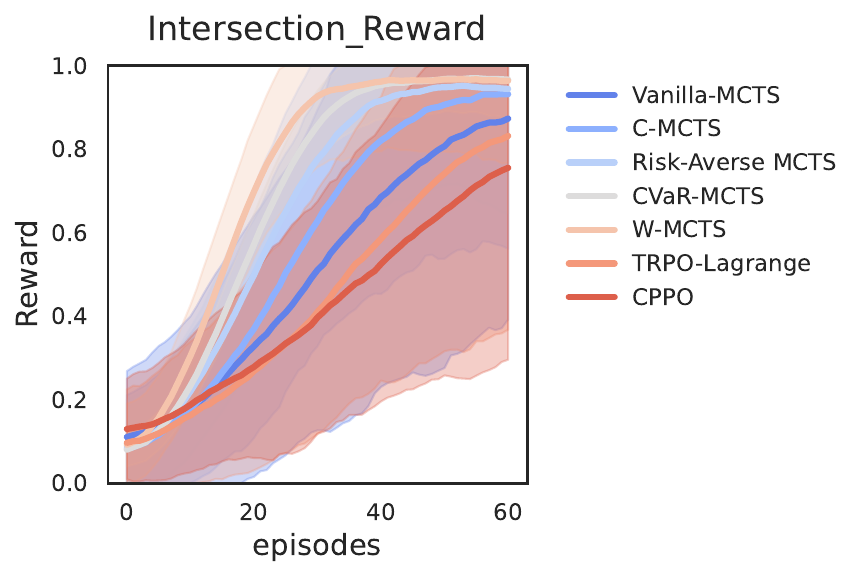}
  \label{fig:all_intersection_r}
  }\hfill
  \subfigure[Racetrack reward]{\includegraphics[width=0.24\textwidth]{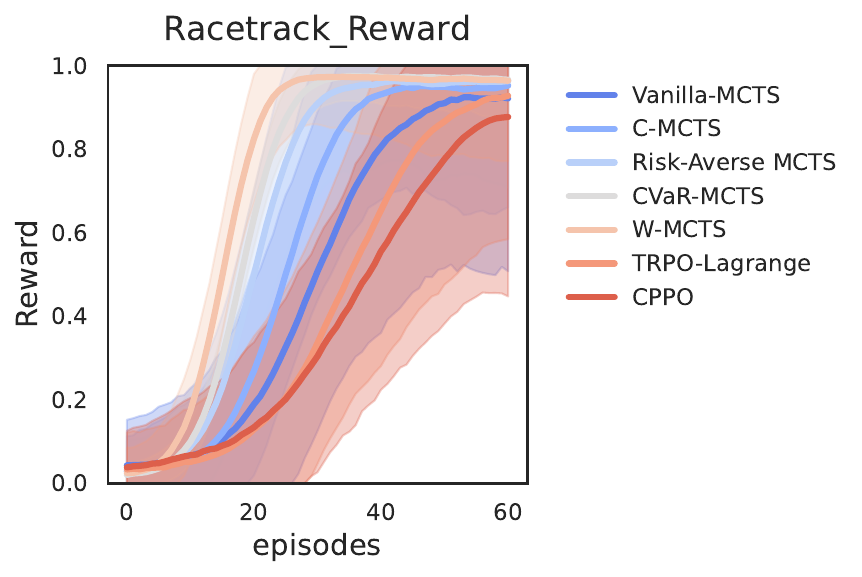}
  \label{fig:all_racetrack_r}
  }\hfill
  \subfigure[Roundabout reward]{\includegraphics[width=0.24\textwidth]{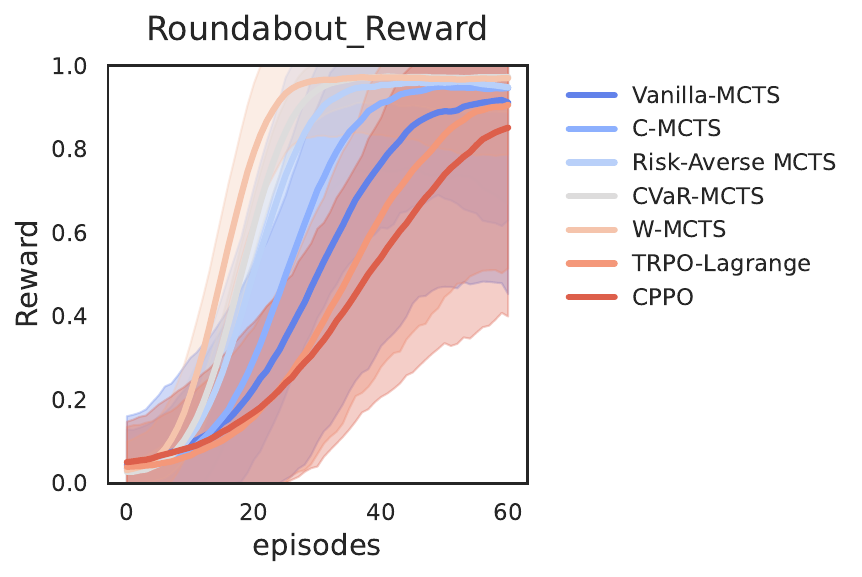}
  \label{fig:all_roundabout_r}
  }\hfill
  \\
  \subfigure[Highway cost]{\includegraphics[width=0.24\textwidth]{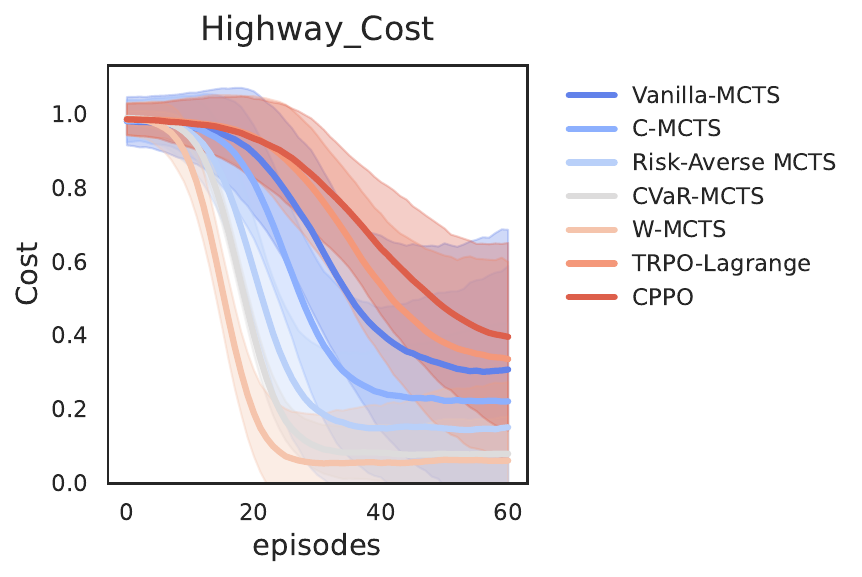}
  \label{fig:all_highway_c}
  }\hfill
  \subfigure[Intersection cost]{\includegraphics[width=0.24\textwidth]{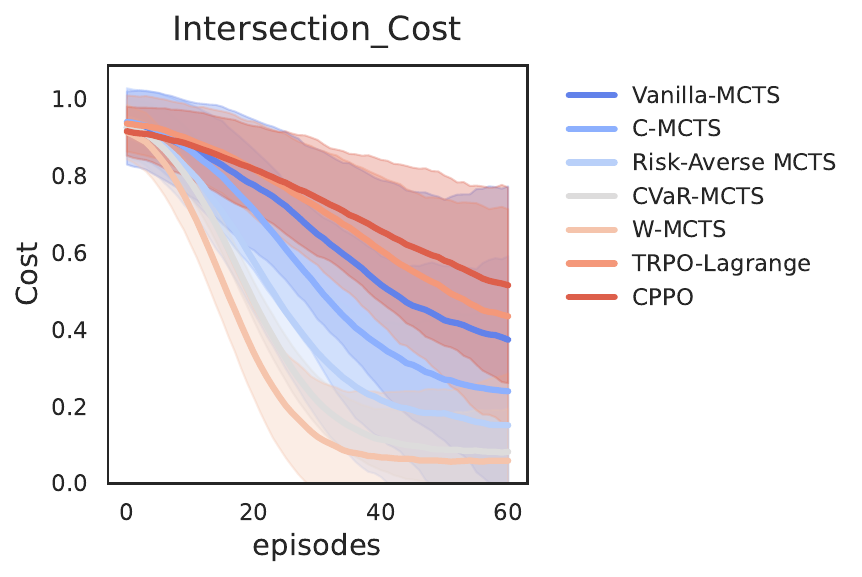}
  \label{fig:all_intersection_c}
  }\hfill
  \subfigure[Racetrack cost]{\includegraphics[width=0.24\textwidth]{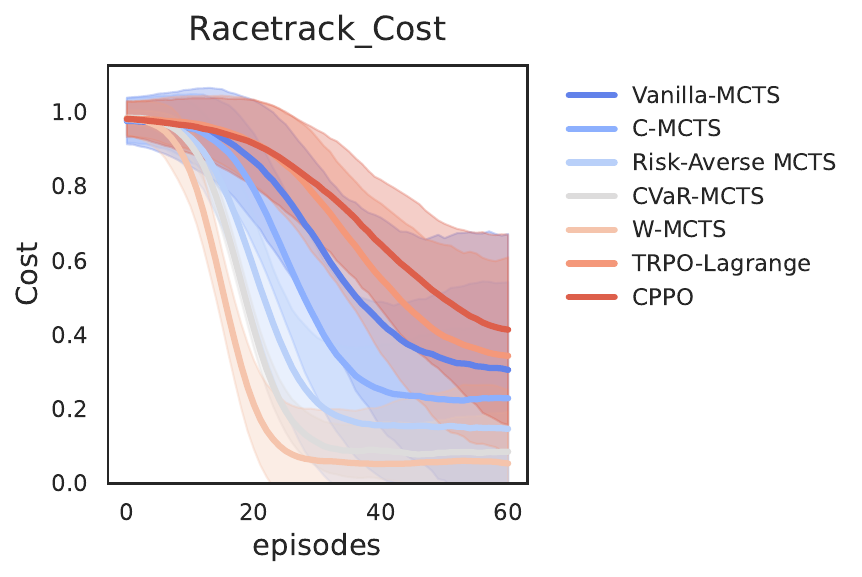}
  \label{fig:all_racetrack_c}
  }\hfill
  \subfigure[Roundabout cost]{\includegraphics[width=0.24\textwidth]{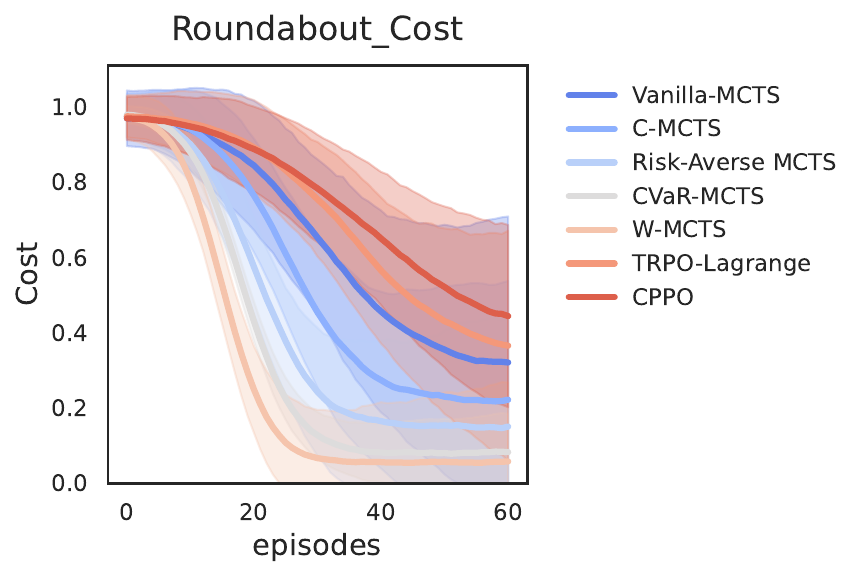}
  \label{fig:all_roundabout_c}
  }\hfill
  \\
  \subfigure[Highway CVaR]{\includegraphics[width=0.24\textwidth]{data/highway/Highway_CVaR.pdf}
  \label{fig:all_highway_cvar}
  }\hfill
  \subfigure[Intersection CVaR]{\includegraphics[width=0.24\textwidth]{data/highway/Intersection_CVaR.pdf}
  \label{fig:all_intersection_cvar}
  }\hfill
  \subfigure[Racetrack CVaR]{\includegraphics[width=0.24\textwidth]{data/highway/Racetrack_CVaR.pdf}
  \label{fig:all_racetrack_cvar}
  }\hfill
  \subfigure[Roundabout CVaR]{\includegraphics[width=0.24\textwidth]{data/highway/Roundabout_CVaR.pdf}
  \label{fig:all_roundabout_cvar}
  }\hfill
  \caption{Performance trends for seven algorithms in four traffic environments.  
  (a–d) show how the average reward evolves over episodes;  
  (e–h) show the corresponding average cost;  
  (i–l) show the 90\%‑CVaR of cumulative cost.  
  All curves are plotted as mean $\pm$ SEM across 30 independent runs.}
  \label{fig:all}
} 
\end{figure*}